%% file: main.tex
\documentclass{article}
\usepackage{fullpage}
\usepackage{charter}
\usepackage{hyperref}
\usepackage[square,numbers]{natbib}
\usepackage{authblk}

\usepackage{url}

\usepackage{amsfonts}
\usepackage{amsmath}
\usepackage{amsthm}
\usepackage{amssymb}

\usepackage{algorithm}
\usepackage{algpseudocode}
\usepackage{bbm}
\usepackage{xspace}
\usepackage[shortlabels]{enumitem}
\usepackage{tikz}
\usetikzlibrary{positioning, arrows, matrix, calc}
\usepackage{pgfplots}
\usepackage{subcaption}
\usepackage{makecell}
\usepackage{resizegather}

\usepackage{multirow}
 
\input{commands}

\title{Data Debugging with Shapley Importance over End-to-End Machine Learning Pipelines \\ \vspace{0.5em}
\large Data Importance and Valuation Meet Feature Extractors and Data Provenance}

\author[1]{Bojan Karlaš}
\author[1]{David Dao}
\author[2]{Matteo Interlandi}
\author[3]{Bo Li}
\author[4]{Sebastian Schelter}
\author[2]{Wentao Wu}
\author[1]{Ce Zhang}
\affil[1]{ETH Zurich, Switzerland}
\affil[2]{Microsoft, USA}
\affil[3]{UIUC, USA}
\affil[4]{University of Amsterdam, Netherlands}
\affil[ ]{\textit {\{bojan.karlas, david.dao, ce.zhang\}@inf.ethz.ch, \{matteo.interlandi, wentao.wu\}@microsoft.com, lbo@uiuc.edu, s.schelter@uva.nl}}

\date{}

\newtheorem{theorem}{Theorem}[section]
\newtheorem{corollary}{Corollary}[theorem]
\newtheorem{lemma}[theorem]{Lemma}

\begin{document}

\setcitestyle{numbers}

\maketitle

\input{abstract}

\input{introduction}

\input{preliminaries}


\input{problem}

\input{approach}
\input{framework}

\input{evaluation}

\input{related}
\input{conclusion}

\bibliographystyle{ACM-Reference-Format}
\bibliography{main}

\appendix 

\input{implementation}

\input{evaluation-appendix}

\input{proofs}

\end{document}

%% file: commands.tex
\algnewcommand{\Inputs}[0]{\State \textbf{inputs}}
\algnewcommand{\Input}[1]{\State \hspace*{\algorithmicindent} #1}
\algnewcommand{\Outputs}[0]{\State \textbf{outputs}}
\algnewcommand{\Output}[1]{\State \hspace*{\algorithmicindent} #1}
\algrenewcommand{\Return}[1]{\State \textbf{return} #1}
\algnewcommand{\Begin}{\State \hspace*{-\algorithmicindent}\textbf{begin}}

\newcommand{\inlinesection}[1]{\medskip \noindent \textbf{#1}}
\newcommand{\sysname}{\texttt{DataScope}\xspace}

\usepackage[T1]{fontenc}

\definecolor{myyellow}{HTML}{FAC802}
\definecolor{myblue}{HTML}{2BBFD9}
\definecolor{myred}{HTML}{DF362A}
\definecolor{mygreen}{HTML}{8AB365}
\definecolor{mypink}{HTML}{C670D2}

\newcommand{\topicitem}[1]{{\small \textit{ \color{cyan}{> #1 }}}}

\newcommand{\citeX}{{\color{red}[CITE]}\xspace}

\usepackage{enumitem}
\setitemize{noitemsep,topsep=0pt,parsep=0pt,partopsep=0pt}

\newcommand{\at}[1]{\protect\ensuremath{\mathsf{#1}}\xspace}

\definecolor{darkorange}{rgb}{1.0, 0.55, 0.0}

\newcommand{\todo}[1]{{\color{red}(TODO: #1)}}

\newcommand{\ww}[1]{\textcolor{blue}{Wentao: #1}}
\newcommand{\bk}[1]{\textcolor{myblue}{Bojan: #1}}
\newcommand{\ssc}[1]{\textcolor{purple}{SSC: #1}}

\usepackage{listings}
\definecolor{darkgray}{rgb}{0.33, 0.33, 0.33}

\lstnewenvironment{Python}[1][]
  {\lstset{language=Python,
    basicstyle      = \scriptsize\ttfamily,
    keywordstyle    = \color{blue},
    keywordstyle    = [2] \color{teal}, 
    stringstyle     = \color{magenta},
    literate={ü}{{\"u}}1 {ö}{{\"o}}1 {É}{{\'E}}1 {œ}{{\oe}}1,
    commentstyle    = \color{darkgray}\ttfamily,
           morekeywords=as,
           morekeywords=with,
           #1}%
  }
  {}

\makeatletter
\newenvironment{customlegend}[1][]{%
    \begingroup
    \let\addlegendimage=\pgfplots@addlegendimage
    \let\addlegendentry=\pgfplots@addlegendentry
    \pgfplots@init@cleared@structures
    \pgfplotsset{#1}%
}{%
    \pgfplots@createlegend
    \endgroup
}%
\makeatother

%% file: abstract.tex

\begin{abstract}
Developing modern machine learning (ML) applications is 
\textit{data-centric}, of which one fundamental challenge is to understand the influence of data quality to ML training ---
\textit{``Which training examples are `guilty' in making the trained ML model predictions inaccurate or unfair?''} 
Modeling data influence for ML training has attracted intensive interest over the last decade, and one popular framework is to compute the \textit{Shapley value} of each training example with respect to utilities such as validation accuracy and fairness of the trained ML model. Unfortunately, despite recent intensive 
interests and research, existing methods only consider  a single ML model ``in isolation'' and do not consider an end-to-end ML pipeline that consists of \textit{data transformations}, \textit{feature extractors}, and \textit{ML training}.

We present \texttt{Ease.ML/DataScope}, the first system that efficiently computes Shapley values of training examples over an \emph{end-to-end} ML pipeline, and illustrate its applications in data debugging for ML training.
To this end, we first develop a novel algorithmic framework  that computes Shapley value over a specific family of ML pipelines that we call \textit{canonical pipelines}: \textit{a positive relational algebra query followed by a $K$-nearest-neighbor (KNN) classifier}. 
We show that, for many subfamilies of canonical pipelines, computing Shapley value is in \textsf{PTIME}, contrasting the exponential complexity of computing Shapley value in general.
We then put this to practice --- given an \texttt{sklearn} pipeline, we approximate it with a canonical pipeline to use as a proxy. 
We conduct extensive experiments illustrating different use cases and utilities.
Our results show that \texttt{DataScope} is up to four orders of magnitude faster over state-of-the-art Monte Carlo-based methods, while being comparably, and often even more, effective in data debugging.
\end{abstract}

{\small {\bf Code Availability: \href{https://github.com/easeml/datascope}{github.com/easeml/datascope}; \href{https://github.com/schelterlabs/arguseyes}{github.com/schelterlabs/arguseyes/tree/datascope}}}

%% file: introduction.tex

%
\begin{Python}[float=tp,frame=none,captionpos=b,texcl=true,numbers=left,xleftmargin=.25in,caption={
A simplified illustration of the core functionality enabled by \sysname --- given an end-to-end ML pipeline (Line 1-19), 
and a utility (e.g., \texttt{sklearn.metrics.accuracy}),
\sysname computes the Shapley value of
each training example as its \textit{importance} with respect to the given utility. 
},label={lst:example}]
# Data loading
train_data = pd.read_csv("...") 
test_data = pd.read_csv("...") 
side_data = pd.read_csv("...")
# Data integration
train_data = train_data.join(side_data, on="item_id")
test_data = test_data.join(side_data, on="item_id")
# Declaratively defined (nested) feature encoding pipeline
pipeline = Pipeline([
  ('features', ColumnTransformer([
    (StandardScaler(), ["num_att1", "num_att2"]),
    (Pipeline([SimpleImputer(), OneHotEncoder()])), 
      ["cat1", "cat2"]),
    (HashingVectorizer(n_features=100), "text_att1")])),
  # ML model for learning  
  ('learner', SVC()])
# Train and evaluate model
pipeline.fit(train_data, train_data.label)
print(pipeline.score(test_data, test_data.label))

# ********************** DATASCOPE **********************
# Run data importance computation over pipeline
(train_data_with_imp, side_data_with_imp) = \
  DataScope.debug(pipeline, sklearn.metrics.accuracy)
# *******************************************************
\end{Python}


\section{Introduction}

Last decade has witnessed the rapid advancement  of machine learning (ML), along which comes the advancement of \textit{machine learning systems}~\cite{Ratner2019-kt}. Thanks to these advancements, training a machine learning model has never been easier today for practitioners --- distributed learning over hundreds of devices~\cite{Liu2020-hb,Li2020-xr,Gan2021-kk, Sergeev2018-om, Jiang2020-qu}, tuning hyper-parameters and selecting the best model~\cite{noauthor_undated-ap, Zoph2016-bz,Feurer2015-cl}, all of which become much more systematic and less mysterious.
Moreover, all major cloud service providers now support AutoML and other model training and serving  services.

\vspace{0.5em}
\noindent 
{\bf \em Data-centric Challenges and Opportunities.} Despite these great advancements, a new collection of challenges start to emerge in building better machine learning applications.
One observation getting great attention recently is that 
\textit{the quality of 
a model is often a reflection 
of the quality of the underlying
training data}. As a result,
often the most practical and
efficient way of improving
ML model quality is to 
improve data quality.
As a result,
recently, researchers have
studied how to conduct 
data cleaning~\cite{Krishnan_undated-mf,karlas2020nearest}, data debugging~\cite{koh2017understanding, koh2019accuracy,ghorbani2019data,jia2019towards,Jia2021-zf, Jia2019-kz}, and data acquisition~\cite{Ratner2017-aw}, specifically for 
the purpose of improving an ML model.

\vspace{0.3em}
\noindent
{\bf \em Data Debugging via Data Importance.}
In this paper, we focus on 
the fundamental problem of 
reasoning about the 
\textit{importance of 
training examples with respect to some utility functions 
(e.g., validation accuracy and fairness) of the trained ML model.} There have been intensive
recent interests to develop methods for 
reasoning about data importance. These 
efforts can be categorized into two 
different views. The \textit{Leave-One-Out (LOO)}
view of this problem tries to calculate,
given a training set $\mathcal{D}$,
the importance of a data example $x \in \mathcal{D}$ modeled as the \emph{utility} decrease after removing this data example: $U(\mathcal{D}) - U(\mathcal{D} \backslash x)$.
To scale-up this process over a large dataset,
researchers have been developing approximation 
methods such as \textit{influence function}
for a diverse set of ML models~\cite{koh2017understanding}.
On the other hand, the \textit{Expected-Improvement (ExpI)} view
of this problem tries to 
calculate such a utility decrease over 
\textit{all possible subsets of $\mathcal{D}$}.
Intuitively, this line of work models data 
importance as an ``expectation'' over 
all possible subsets/sub-sequences of $\mathcal{D}$, instead of trying to reason about it solely on a single 
training set. 
One particularly popular approach is 
to use Shapley value~\cite{ghorbani2019data,jia2019towards,Jia2021-zf},
a concept in game theory that has been 
applied to data importance 
and data valuation~\cite{Jia2019-kz}.

\vspace{0.3em}
\noindent
{\bf \em Shapley-based Data Importance.}
In this paper, we do not champion one 
view over the other (i.e., LOO vs. ExpI).
We scope ourselves and only focus on 
Shapley-based methods since 
previous work has shown applications 
that can only use 
Shapley-based methods because of the favorable
properties enforced by the Shapley value.
Furthermore, taking expectations
can sometimes provide a more reliable
importance measure~\cite{Jia2019-kz} than
simply relying on a single dataset.
Nevertheless,
we believe that it is important 
for future ML systems to support both 
and we hope that this paper
can inspire future research in data importance for both the LOO and ExpI views.

One key challenge of Shapley-based
data importance is its computational 
complexity --- in the worst case,
it needs to enumerate \textit{exponentially} 
many subsets. There have been different
ways to \emph{approximate} this computation, either
with MCMC~\cite{ghorbani2019data} and group testing~\cite{jia2019towards} or 
proxy models such as K-nearest neighbors (KNN)~\cite{Jia2021-zf}.
One surprising result is that 
Shapley-based data importance 
can be calculated efficiently (in 
\textit{polynomial} time) for 
KNN classifiers~\cite{Jia2021-zf}, and
using this as a proxy for 
other classifiers performs well
over a diverse range of tasks~\cite{Jia2019-kz}.

\vspace{0.3em}
\noindent
{\bf \em Data Importance over Pipelines.}
Existing methods for computing Shapley values~\cite{ghorbani2019data,jia2019towards,Jia2021-zf,Jia2019-kz} are designed to directly operate on a single numerical input dataset for an ML model, typically in matrix form. However, in real-world ML applications, this data is typically generated on the fly from multiple data sources with an ML pipeline. Such pipelines often take multiple datasets as input, and transform them into a single numerical input dataset with relational operations (such as joins, filters, and projections) and common feature encoding techniques, often based on nested estimator/transformer pipelines, which are integrated into popular ML libraries such as scikit-learn~\cite{pedregosa2011scikit}, SparkML~\cite{meng2016mllib} or Google~TFX~\cite{baylor2017tfx}. It is an open problem how to apply Shapley-value computation in such a setup.

\autoref{lst:example} shows a toy example of such an end-to-end ML pipeline, which includes relational operations from pandas for data preparation (lines~3-9), a nested estimator/transformer pipeline for encoding numerical, categorical, and textual attributes as features (lines~12-16), and an ML model from scikit-learn (line~18). The code loads the data, splits it temporally into training and test datasets, `fits' the pipeline to train the model, and evaluates the predictive quality on the test dataset. This leads us to the key question we pose in this work: 
\begin{quote}
\textit{Can we efficiently 
compute Shapley-based data importance
over such an end-to-end ML pipeline with \underline{both} data processing and 
ML training?}    
\end{quote}

\vspace{0.3em}
\noindent
{\bf \em Technical Contributions.}
We present 
\texttt{Ease.ML}/\sysname, the first 
system 
that efficiently computes
and approximates Shapley
value over end-to-end ML pipelines. 
\sysname takes as input
an ML pipeline (e.g., 
a \texttt{sklearn} pipeline)
and a given utility function,
and outputs 
the importance, measured
as the Shapley value, of each input tuple of the ML pipeline.
\autoref{lst:example} (lines 21-25) gives a simplified illustration of this core functionality provided by \sysname.
A user points \sysname to the pipeline code, and \sysname executes the pipeline, extracts the input data, which is annotated with the corresponding Shapley value per input tuple. The user could then, for example, retrieve and inspect the least useful input tuples. 
We present several use cases of how these importance values can be used, including label denoising and 
fairness debugging, in \autoref{sec:evaluation}.
We made the following contributions when developing \sysname.


\vspace{0.3em}
\textbf{Our first technical contribution}
is to jointly analyze Shapley-based 
data importance together with a \textit{feature extraction pipeline}. To our best knowledge,
this is the first time that 
these two concepts are 
analyzed together.
We first show 
that we can develop a \textit{PTIME algorithm} 
given a counting oracle relying on data provenance.
We then show that,
for a collection of
``canonical pipelines'', which covers many real-world 
pipelines~\cite{psallidas2019data} (see \autoref{tbl:pipelines} in \autoref{sec:evaluation} for examples), 
this counting oracle itself can be implemented 
in polynomial time. This provides 
an efficient algorithm
for computing Shapley-based 
data importance 
over these ``canonical pipelines''.

\vspace{0.3em}
\textbf{Our second technical contribution} 
is to understand and further adapt 
our technique in the context of real-world ML pipelines. 
We identify scenarios from the aforementioned 500K ML pipelines where our techniques
cannot be directly applied to have PTIME algorithms.
We introduce 
a set of simple yet effective 
approximations and optimizations 
to further improve the performance on these scenarios.

\vspace{0.3em}
\textbf{Our third technical contribution}
is an extensive empirical study of 
\sysname. We show that for a
diverse range of ML pipelines, 
\sysname
provides effective approximations to support
a range of applications to improve 
the accuracy and fairness
of an ML pipeline. 
Compared with strong
state-of-the-art methods based on 
Monte Carlo sampling,
\sysname can be up to four orders of magnitude faster while being comparably, and often even more, effective in data debugging.

%% file: preliminaries.tex
\section{Preliminaries}
\label{sec:prelim}

In this section we describe several concepts from existing research that we use as basis for our contributions. Specifically, (1) we present the definition of machine learning pipelines and their semantics, and (2) we describe decision diagrams as a tool for compact representation of Boolean functions.

\subsection{End-to-end ML Pipelines} \label{sec:end-to-end-ml-pipelines}


An end-to-end ML application consists 
of two components: (1) a feature 
extraction pipeline, and (2) a
downstream ML model. To conduct 
a joint analysis over one such 
end-to-end application, 
we leave the precise definite to \autoref{sec:problem-formal}.
One important component in our 
analysis relies on the \textit{provenance} of 
the feature extraction pipeline,
which we will discuss as follows.

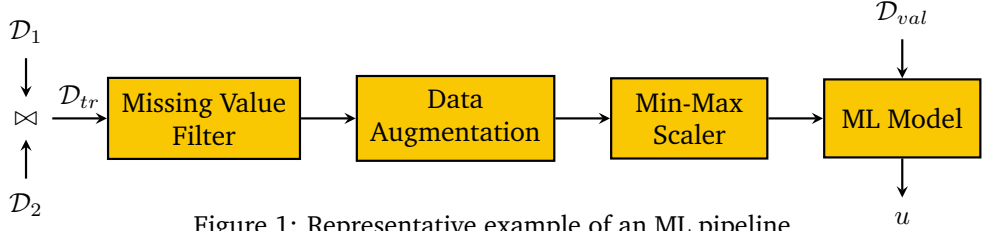
\begin{figure}
    \centering
    \resizebox{0.8\columnwidth}{!}{%
    \begin{tikzpicture}[align=center, node distance=5mm and 7mm, line width=0.8pt]
        \tikzstyle{free} = [inner sep=5pt]
        \tikzstyle{box} = [draw, rectangle, fill=myyellow, inner sep=5pt, minimum width=2cm, minimum height=1cm]
        
        \node[free] (join) {$\bowtie$};
        \node[free] (dataset1) [above=of join] {$\mathcal{D}_1$};
        \node[free] (dataset2) [below=of join] {$\mathcal{D}_2$};
        
        \node[box] (o1) [right=of join] {Missing Value \\ Filter};
        \node[box] (o2) [right=of o1] {Data \\ Augmentation};
        \node[box] (o3) [right=of o2] {Min-Max \\ Scaler};
        \node[box] (model) [right=of o3] {ML Model};
        
        \node[free] (val_dataset) [above=of model] {$\mathcal{D}_{val}$};
        \node[free] (acc) [below=of model] {$u$};

        \draw[-stealth] (dataset1) -- (join);
        \draw[-stealth] (dataset2) -- (join);
        
        \draw[-stealth] (join) -- node[above] {$\mathcal{D}_{tr}$} (o1);
        \draw[-stealth] (o1) -- (o2);
        \draw[-stealth] (o2) -- (o3);
        \draw[-stealth] (o3) --  (model);
        \draw[-stealth] (val_dataset) -- (model);
        \draw[-stealth] (model) -- (acc);
    \end{tikzpicture}
    }
    \vspace{-2em}
    \caption{Representative example of an ML pipeline.}
    \vspace{2em}
    \label{fig:example-pipeline}
    \end{figure}

\inlinesection{Provenance Tracking.}
Input examples (tuples) in $\mathcal{D}_{tr}$
are transformed by a feature processing pipeline before being turned into a 
processed training dataset $\mathcal{D}_{tr}^f := f(\mathcal{D}_{tr})$, which is directly used to train the model. To enable ourselves to compute the importance 
of examples in  $\mathcal{D}_{tr}$,
it is useful to relate the presence of tuples in $\mathcal{D}_{tr}^f$ in the training dataset with respect to the presence of tuples in $\mathcal{D}_{tr}$. In other words, we need to know the \emph{provenance} of training tuples. In this paper, we
rely on the well-established theory of provenance semirings \cite{green2007provenance} to
describe such provenance.

We associate a variable $a_t \in A$ with every tuple $t$ in the training dataset $\mathcal{D}_{tr}$.
We define \emph{value assignments} $v : A \rightarrow \mathbb{B}$ to 
describe whether a given tuple $t$
appears in $\mathcal{D}_{tr}$ --- by 
setting $v(a_t)=0$, we ``exclude'' $t$
from $\mathcal{D}_{tr}$ and by setting
$v(a_t)=1$, we ``include''
$t$ in $\mathcal{D}_{tr}$.
Let $\mathcal{V}_A$
be the set of all possible such 
value assignments ($|\mathcal{V}_A| = 2^{|A|}$).
We use
\[
\mathcal{D}_{tr}[v] = \{t \in \mathcal{D}_{tr} | v(a_t) \ne 0\}
\]
to denote a subset of training 
examples, only containing tuples 
$t$ whose corresponding variable in $a_t$ is
set to 1 according to $v$.

To describe the association between
$\mathcal{D}_{tr}$ and its transformed
version $\mathcal{D}_{tr}^f$, we
annotate each potential tuple in 
$\mathcal{D}_{tr}^f$ with an attribute $p : \mathbb{D} \rightarrow \mathbb{B}[A]$ containing its \emph{provenance polynomial}~\cite{green2007provenance} which is a logical formula with variables in $A$ and binary coefficients (e.g. $a_1 + a_2 \cdot a_3$) --- $p(t)$
is true only if tuple $t$
appears in $\mathcal{D}_{tr}^f$. For such polynomials, an \emph{addition} corresponds to a \emph{union} operator in the ML pipeline, and a \emph{multiplication} corresponds to a \emph{join} operator in the pipeline.
\autoref{fig:patterns}
illustrates some examples
of the association between
$a_t$ and $p(t)$.

Given a value assignment $v \in \mathcal{V}_{\mathcal{A}}$, we can define an evaluation function $\mathrm{eval}_v \ \phi$ that returns the \emph{evaluation} of a
provenance polynomial $\phi$
under the assignment $v$. Given a value assignment $v$, we can obtain the corresponding \emph{transformed dataset} by evaluating all provenance polynomials of its tuples, as such:
\begin{equation} \label{eq:candidate-dataset}
    \mathcal{D}_{tr}^f[v] := \{ t \in \mathcal{D}_{tr}^f \ | \ \mathrm{eval}_v (p(t)) \neq 0 \}
\end{equation}
Intuitively, $\mathcal{D}_{tr}^f[v]$
corresponds to the result of applying
the feature transformation $f$
over a subset of training examples
that only contains tuples $t$ whose 
corresponding variable $a_t$ is
set to 1.

Using this approach, given a feature processing pipeline $f$ and a value assignment $v$, we can obtain the
transformed training set $\mathcal{D}_{tr}^f[v] = f(\mathcal{D}_{tr}[v])$. 

\subsection{Additive Decision Diagrams (ADD's)} \label{sec:additive-decision-diagrams}

\inlinesection{Knowledge Compilation.}
Our approach of computing the Shapley value will rely upon being able to construct functions over Boolean inputs $\phi : \mathcal{V}_A \rightarrow \mathcal{E}$, where $\mathcal{E}$ is some finite \emph{value set}. We require an elementary algebra with $+$, $-$, $\cdot$ and $/$ operations to be defined for this value set. Furthermore, we require this value set to contain a \emph{zero element} $0$, as well as an \emph{invalid element} $\infty$ representing an undefined result (e.g. a result that is out of bounds).
We then need to count the number of value assignments $v \in \mathcal{V}_A$ such that $\phi(v) = e$, for some specific value $e \in \mathcal{E}$. This is referred to as the \emph{model counting} problem, which is \#\textsf{P} complete for arbitrary logical formulas~\cite{Valiant1979-jp,arora2009computational}. For example, if $A = \{a_1, a_2, a_3\}$, we can define $\mathcal{E} = \{0,1,2,3, \infty\}$ to be a value set and a function $\phi(v) := v(a_1) + v(a_2) + v(a_3)$ corresponding to the number of variables in $A$ that are set to $1$ under some value assignment $v \in \mathcal{V}_A$.

\emph{Knowledge compilation}~\cite{cadoli1997survey} has been developed as a well-known approach to tackle this model counting problem.
It was also successfully applied to various problems in data management~\cite{Jha2011}. 
One key result from this line of work is that, if we can construct certain polynomial-size data structures to represent our logical formula, then we can perform model counting in polynomial time. Among the most notable of such data structures are \emph{decision diagrams}, specifically binary decision diagrams~\cite{lee1959bdd,bryant1986bdd} and their various derivatives~\cite{bahar1997algebric,sanner2005affine, lai1996formal}.
For our purpose in this paper, we use the \emph{additive decision diagrams} (ADD), as detailed below.

\vspace{-0.5em}
\inlinesection{Additive Decision Diagrams (ADD).}
We define a simplified version of the \emph{affine algebraic decision diagrams}~\cite{sanner2005affine}.
An ADD is a directed acyclic graph defined over a set of nodes $\mathcal{N}$ and a special \emph{sink node} denoted $\boxdot$. Each node $n \in \mathcal{N}$ is associated with a variable $a(n)\in A$.
Each node has two outgoing edges, 
$c_L(n)$ and $c_H(n)$, that point to its \emph{low} and \emph{high} child nodes, respectively. 
For some value assignment $v$, the low/high edge corresponds to $v(a)=0$/$v(a)=1$. Furthermore, each low/high edge is associated with an increment $w_L$/$w_H$ that maps edges to elements of $\mathcal{E}$.

Note that each node $n \in \mathcal{N}$ represents the root of a subgraph and defines a Boolean function. Given some value assignment $v \in \mathcal{V}_{A}$ we can evaluate this function by constructing a path starting from $n$ and at each step moving towards the low or high child depending on whether the corresponding variable is assigned a $0$ or $1$. The value of the function is the result of adding all the edge increments together. 
\autoref{fig:example-add-structure} presents an example ADD with one path highlighted in red.
Formally, we can define the evaluation of the function defined by the node $n$ as follows:
\begin{equation} \label{eq:dd-eval-definition}
    \mathrm{eval}_v (n) := \begin{cases}
        0,                    & \mathrm{if} \ n = \boxdot, \\
        w_L(n) + \mathrm{eval}_v (c_L(n))     & \mathrm{if} \ v(x(n)) = 0, \\
        w_H(n) + \mathrm{eval}_v (c_H(n))     & \mathrm{if} \ v(x(n)) = 1. \\
    \end{cases}
\end{equation}
In our work we focus specifically on ADD's that are \emph{full} and \emph{ordered}. A diagram is full if every path from root to sink encounters every variable in $A$ exactly once. On top of that, an ADD is ordered when on each path from root to sink variables always appear in the same order. For this purpose, we define $\pi : A \rightarrow \{1,...,|A|\}$ to be a permutation of variables that assigns each variable $a \in A$ an index.

\begin{figure}
    \centering
    \begin{subfigure}[t]{.45\linewidth}
        \centering
        \caption{ADD}\label{fig:example-add-structure}
            \begin{tikzpicture}[align=center, node distance=5mm and 5mm, line width=0.8pt] 
                \tikzstyle{free} = [inner sep=5pt]
                \tikzstyle{var} = [draw, rectangle, inner sep=2pt, minimum width=4mm, minimum height=4mm]
                \tikzstyle{root} = [draw, rectangle, inner sep=2pt, minimum width=4mm, minimum height=4mm]
                
                \tikzstyle{path} = [draw=myred]
                
                
                \node[var] (x11-1) {};
                \draw (x11-1 -| -2,0) node[anchor=west] {$a_{1,1}$};
                
                
                \node[var] (x21-1) [below left=of x11-1] {};
                \draw[-stealth, dashed] (x11-1) to [bend right] (x21-1);
                
                \node[var] (x21-2) [below right=of x11-1] {};
                \draw[-stealth, path] (x11-1) to [bend left] (x21-2);
                
                \draw (x21-1 -| -2,0) node[anchor=west] {$a_{2,1}$};
                
                
                \node[var] (x22-1) [below=of x21-1] {};
                \draw[-stealth, dashed] (x21-1) to [bend right] (x22-1);
                \draw[-stealth] (x21-1) to [bend left] (x22-1);
                
                \node[var] (x22-2) [below=of x21-2] {};
                \draw[-stealth, dashed, path] (x21-2) to [bend right] (x22-2);
                \draw[-stealth] (x21-2) to [bend left] node[right] {$+1$} (x22-2);
                
                \draw (x22-1 -| -2,0) node[anchor=west] {$a_{2,2}$};
                
                
                \node[var] (x12-1) [below right=of x22-1] {};
                \draw[-stealth, dashed] (x22-1) to [bend right] (x12-1);
                \draw[-stealth] (x22-1) to [bend left] (x12-1);
                \draw[-stealth, dashed] (x22-2) to [bend right] (x12-1);
                \draw[-stealth, path] (x22-2) to [bend left] node[right] {$+1$} (x12-1);
                
                \draw (x12-1 -| -2,0) node[anchor=west] {$a_{1,2}$};
                
                
                \node[var] (x23-1) [below left=of x12-1] {};
                \draw[-stealth, dashed] (x12-1) to [bend right] (x23-1);
                
                \node[var] (x23-2) [below right=of x12-1] {};
                \draw[-stealth, path] (x12-1) to [bend left] (x23-2);
                
                \draw (x23-1 -| -2,0) node[anchor=west] {$a_{2,3}$};
                
                
                \node[root] (xs) [below right=of x23-1] {$\cdot$};
                \draw[-stealth, dashed] (x23-1) to [bend right] (xs);
                \draw[-stealth] (x23-1) to [bend left] (xs);
                \draw[-stealth, dashed] (x23-2) to [bend right] (xs);
                \draw[-stealth, path] (x23-2) to [bend left] node[right] {$+1$} (xs);
            \end{tikzpicture}
    \end{subfigure}
    \begin{subfigure}[t]{.45\linewidth}
        \centering
        \caption{Uniform ADD}\label{fig:example-uniform-add}
    \begin{tikzpicture}[align=center, node distance=5mm and 5mm, line width=0.8pt] 
        \tikzstyle{free} = [inner sep=5pt]
        \tikzstyle{var} = [draw, rectangle, inner sep=2pt, minimum width=4mm, minimum height=4mm]
        \tikzstyle{root} = [draw, rectangle, inner sep=2pt, minimum width=4mm, minimum height=4mm]
        
        
        \node[var] (x1) {};
        
        \draw (x1 -| -1,0) node[anchor=west] {$a_{1}$};
        
        
        \node[var] (x2) [below=of x1] {};
        \draw[-stealth, dashed] (x1) to [bend right] (x2);
        \draw[-stealth] (x1) to [bend left] node[right] {$+5$} (x2);
        
        \draw (x2 -| -1,0) node[anchor=west] {$a_{2}$};
        
        
        \node[var] (x3) [below=of x2] {};
        \draw[-stealth, dashed] (x2) to [bend right] (x3);
        \draw[-stealth] (x2) to [bend left] node[right] {$+5$} (x3);
        
        \draw (x3 -| -1,0) node[anchor=west] {$a_{3}$};
        
        
        \node[root] (xs) [below=of x3] {$\cdot$};
        \draw[-stealth, dashed] (x3) to [bend right] (xs);
        \draw[-stealth] (x3) to [bend left] node[right] {$+5$} (xs);
    \end{tikzpicture}
    \end{subfigure}
    \vspace{-1em}
    \caption{(a) An ordered and full ADD for computing $\phi(v) := v(a_{1,1}) \cdot \big( v(a_{2,1}) + v(a_{2,2}) \big) + v(a_{1,2}) \cdot v(a_{2,3}).$
    (b) A uniform ADD 
    for computing $\phi(v) := 5 \cdot (v(a_1) + v(a_2) + v(a_3))$.
    }
    \vspace{2em}
\end{figure}

\inlinesection{Model Counting.}
We define a model counting operator 
\begin{equation} \label{eq:dd-count-definition}
    \mathrm{count}_e (n) := \Big| \Big\{ v \in \mathcal{V}_{A [\leq \pi(a(n))]} \ | \ \mathrm{eval}_v (n) = e \Big\} \Big|,
\end{equation}
where $A [\leq \pi(a(n))]$ is the subset of variables in $A$ that include $x(n)$ and all variables that come before it in the permutation $\pi$.
For an ordered and full ADD, $\mathrm{count}_e (n)$ satisfies the following recursion:
\begin{equation} \label{eq:dd-count-recursion}
\begin{split}
\mathrm{count}_e (n) := \begin{cases}
    1,     & \mathrm{if} \ e=0 \ \mathrm{and} \ n = \boxdot, \\
    0,                    & \mathrm{if} \ e = \infty \ \mathrm{or} \ n = \boxdot, \\
    \mathrm{count}_{e - w_L(n)} (c_L(n)) + \mathrm{count}_{e - w_H(n)} (c_H(n)),     & \mathrm{otherwise}. \\
\end{cases}
\end{split}
\end{equation}
The above recursion can be implemented as a dynamic program with computational complexity $O(|\mathcal{N}| \cdot |\mathcal{E}|)$.

In special cases when the ADD is structured as a chain with one node per variable, all low increments equal to zero and all high increments equal to some constant $E \in \mathcal{E}$, we can perform model counting in constant time. We call this a \emph{uniform} ADD, and \autoref{fig:example-uniform-add} presents an example. 
The $\mathrm{count}_e$ operator for a uniform ADD can be defined as
\begin{equation} \label{eq:dd-count-uniform}
    \mathrm{count}_{e} (n) := \begin{cases}
        \binom{\pi(a(n))}{ e / E},     & \mathrm{if} \ e \ \mathrm{mod} \ E = 0, \\
        0     & \mathrm{otherwise}. \\
    \end{cases}
\end{equation}
Intuitively, if we observe the uniform ADD shown in \autoref{fig:example-uniform-add}, we see that the result of an evaluation must be a multiple of $5$. For example, to evaluate to $10$, the evaluation path must pass a \emph{high} edge exactly twice. Therefore, in a $3$-node ADD with root node $n_R$, the result of $\mathrm{count}_{10} (n_R)$ will be exactly $\binom{3}{2}$.

\inlinesection{Special Operations on ADD's.}
Given an ADD with node set $\mathcal{N}$, we define two operations that will become useful later on when constructing diagrams for our specific scenario:
\begin{enumerate}[leftmargin=*]
    \item \emph{Variable restriction}, denoted as $\mathcal{N}[a_i \gets V]$, which restricts the domain of variables $A$ by forcing the variable $a_i$ to be assigned the value $V$. This operation removes every node $n \in \mathcal{N}$ where $a(n) = a_i$ and rewires all incoming edges to point to the node's high or low child depending on whether $V=1$ or $V=0$.
    \item \emph{Diagram summation}, denoted as $\mathcal{N}_1 + \mathcal{N}_2$, where $\mathcal{N}_1$ and $\mathcal{N}_2$ are two ADD's over the same (ordered) set of variables $A$. 
    It starts from the respective root nodes $n_1$ and $n_2$ and produces a new node $n := n_1 + n_2$. We then apply the same operation to child nodes. Therefore, $c_L(n_1 + n_2) := c_L(n_1) + c_L(n_2)$ and $c_H(n_1 + n_2) := c_H(n_1) + c_H(n_2)$. Also, for the increments, we can define $w_L(n_1 + n_2) := w_L(n_1) + w_L(n_2)$ and $w_H(n_1 + n_2) := w_H(n_1) + w_H(n_2)$.
\end{enumerate}


%% file: problem.tex
\section{Data Importance over ML Pipelines}
\label{sec:problem}

We first recap the problem of computing data importance for ML pipelines in \autoref{sec:problem-importance}, formalise the problem in \autoref{sec:problem-formal}, and outline core technical efficiency and scalability issues afterwards.
We will describe the \sysname approach 
in \autoref{sec:approach}
and our theoretical framework in 
\autoref{sec:framework}.

\subsection{Data Importance for ML Pipelines}
\label{sec:problem-importance}

In real-world ML, one often encounters data-related problems in the input training set (e.g., wrong labels, outliers, biased samples) that lead to sub-optimal quality of the user's model. As illustrated in previous work~\cite{koh2017understanding, koh2019accuracy,ghorbani2019data,jia2019towards,Jia2021-zf, Jia2019-kz}, many data debugging and understanding problems hinge on the following  fundamental question:
\begin{quote}
\em Which data examples in the training set are most important for the model utility ?
\end{quote}

A common approach is to model this problem as computing the {\em Shapley value} of each data example as a measure of its importance to a model, which has been applied to a wide range use cases~\cite{ghorbani2019data,jia2019towards,Jia2021-zf, Jia2019-kz}.
However, this line of work focused solely on ML model training but ignored the \textit{data pre-processing pipeline} prior to model training, which includes steps such as feature extraction, data augmentation, 
etc. This significantly limits its applications to real-world scenarios, most of which consist
of a non-trivial data processing pipeline~\cite{psallidas2019data}.
In this paper, we take the first step in applying
Shapley values to debug end-to-end ML pipelines.

\subsection{Formal Problem Definition}
\label{sec:problem-formal}

We first formally define the core technical problem. 

\inlinesection{ML Pipelines.} Let $\mathcal{D}_{{e}}$ be an input training set for a machine learning task, potentially accompanied by additional relational side datasets $\mathcal{D}_{s_1},\dots,\mathcal{D}_{s_k}$. We assume the data to be in a \emph{star} database schema, where each tuple from a side dataset $\mathcal{D}_{s_i}$ (the ``dimension'' tables) can be joined with multiple tuples from $\mathcal{D}_{{e}}$ (the ``fact'' table).
Let $f$ be a feature extraction pipeline that transforms the relational inputs $\mathcal{D}_{tr} = \{ \mathcal{D}_{e}, \mathcal{D}_{s_1},\dots,\mathcal{D}_{s_k} \}$ into a set of training tuples $\{t_i = (x_i, y_i)\}_{i \in [m]}$ made up of feature and label pairs that the ML training algorithm $\mathcal{A}$ takes as input.
Note that $\mathcal{D}_{e}$ represents \texttt{train\_data} in our toy example in \autoref{lst:example}, $\mathcal{D}_{s}$ represents \texttt{side\_data}, while $f$ refers to the data preparation operations from lines 6-14, and the model $\mathcal{A}$ corresponds to the support vector machine \texttt{SVC} from line~16.

After feature extraction and training, we obtain an ML model:  
\[
\mathcal{A} \circ f (\mathcal{D}_{tr}).
\]
We can measure the \emph{quality} of this model in various ways, e.g., via validation accuracy and a fairness metric. Let $\mathcal{D}_{v}$ be a given set of relational validation data with the same schema as $\mathcal{D}_{e}$. Applying $f$ to $\mathcal{D}_{val} = \{ \mathcal{D}_{v}, \mathcal{D}_{s_1},\dots,\mathcal{D}_{s_k} \}$ produces a set of validation tuples $\{t_i = (\tilde{x}_i, 
\tilde{y}_i)\}_{i \in [p]}$ made up of feature and label pairs, on which we can derive predictions with our trained model $\mathcal{A} \circ f (\mathcal{D}_{tr})$. Based on this, we define a utility function $u$, which measures the performance of the predictions:
\[
u (\mathcal{A} \circ f (\mathcal{D}_{tr}), f(\mathcal{D}_{val})) \mapsto [0, 1].
\]
For readability, we use the following notation in cases where the model $\mathcal{A}$ and pipeline $f$ are clear from context:
\begin{equation} \label{eq:model-acc-definition}
  u (\mathcal{D}_{tr}, \mathcal{D}_{val}) := u (\mathcal{A} \circ f (\mathcal{D}_{tr}), f(\mathcal{D}_{val}))
\end{equation}

\inlinesection{Additive Utilities.}
In this paper, we focus on \textit{additive utilities} that cover the most important set of utility functions in practice (e.g., validation loss, validation accuracy, various fairness metrics, etc.). 
A utility function $u$ is \textit{additive} if there exists a \emph{tuple-wise} utility $u_T$ such that $u$ can be rewritten as
\begin{equation} \label{eq:additive-utility}
u (\mathcal{D}_{tr}, \mathcal{D}_{val})
= 
w \cdot
\sum_{t_{val} \in f(\mathcal{D}_{val})}
   u_T \bigg( \Big(\mathcal{A} \circ f (\mathcal{D}_{tr})\Big)(t_{val}), t_{val} \bigg).
\end{equation}
Here, $w$ is a scaling factor only relying on $\mathcal{D}_{val}$. The tuple-wise utility $u_{T} : (y_{pred}, t_{val}) \mapsto [0, 1]$ takes a validation tuple $t_{val} \in \mathcal{D}_{val}$ as well as a class label $y_{pred} \in \mathcal{Y}$ predicted by the model for 
$t_{val}$. 
It is easy to see that 
popular utilities such as validation accuracy are all additive, e.g., the accuracy utility is simply defined by plugging $u_T(y_{pred}, (x_{val}, y_{val})) := \mathbbm{1} \{y_{pred} = y_{val}\}$ into \autoref{eq:additive-utility}. 

\inlinesection{Example: False Negative Rate as an Additive Utility.} Apart from accuracy which represents a trivial example of an additive utility, we can show how some more complex utilities happen to be additive and can therefore be decomposed according to \autoref{eq:additive-utility}. As an example, we use \emph{false negative rate (FNR)} which can be defined as such:

\begin{equation}
  u(\mathcal{D}_{tr}, \mathcal{D}_{val}) :=
  \frac{\sum_{t_{val} \in f(\mathcal{D}_{val})} \mathbbm{1}\{ (\mathcal{A} \circ f(\mathcal{D}_{tr}))(t_{val}) = 0 \} \mathbbm{1} \{ y(t_{val}) = 1 \} }{|\{ t_{val} \in \mathcal{D}_{val} \ : \ y(t_{val}) = 1 \}|}.
\end{equation}
In the above expression we can see that the denominator only depends on $\mathcal{D}_{val}$ which means it can be interpreted as the scaling factor $w$. We can easily see that the expression in the numerator neatly fits the structure of \autoref{eq:additive-utility} as long as we we define $u_T$ as $u_T (y_{pred}, (x_{val}, y_{val})) := \mathbbm{1} \{ y_{pred} = 0 \} \mathbbm{1} \{ y_{val} = 1 \}$. Similarly, we are able to easily represent various other utilities, including: false positive rate, true positive rate (i.e. recall), true negative rate (i.e. specificity), etc. We describe 
an additional example in \autoref{sec:approach-approx}.

\inlinesection{Shapley Value.} The Shapley value, denoting the importance of an input tuple $t_i$ for the ML pipeline, is defined as
\[
\varphi_i = \frac{1}{|\mathcal{D}_{tr}|} \sum_{S \subseteq \mathcal{D}_{tr} \backslash \{t_i\}} {n - 1 \choose |S|}^{-1} \left( 
    u (S \cup \{t_i\}, \mathcal{D}_{val}) -
    u (S, \mathcal{D}_{val}) 
\right). 
\]
Intuitively, the {\em importance} of $t_i$ over a subset $S \subseteq \mathcal{D}_{tr} \backslash \{t_i\}$ is measured as the difference of the utility $u \circ \mathcal{A} \circ f (S \cup \{t_i\})$ \textit{with} $t_i$ to the utility $u \circ \mathcal{A} \circ f (S)$ \textit{without} $t_i$. 
The Shapley value takes the average of all such possible subsets $S \subseteq \mathcal{D}_{tr} \backslash \{t_i\}$, which allows it to have a range of desired properties that significantly benefit data debugging tasks,
often leading to more effective 
data debugging mechanisms 
compared to other leave-one-out methods.

\subsection{Prior Work and Challenges}

All previous research focuses on the scenario in which there is no ML pipeline $f$ (i.e., one directly works with the vectorised training examples $\{t_i\}$). Even in this case, computing Shapley values is tremendously difficult since its complexity for general ML model is \texttt{\#P}-hard. To accommodate this computational challenge, previous work falls into two categories:

\begin{enumerate}[noitemsep,topsep=0pt,parsep=0pt,partopsep=0pt,topsep=0pt, leftmargin=*]
\item {\em Monte Carlo Shapley}: One natural line of efforts tries to estimate Shapley value with Markov Chain Monte Carlo (MCMC) approaches. This includes vanilla Monte Carlo sampling, group testing~\cite{jia2019towards,Zhou2014-ca}, and truncated Monte Carlo sampling~\cite{ghorbani2019data}. 
\item {\em KNN Shapley}: Even the most efficient Monte Carlo Shapley methods need to train multiple ML models (i.e., evaluate $\mathcal{A}$ multiple times) and thus exhibit long running time for datasets of modest sizes. Another line of research proposes to approximate the model $\mathcal{A}$ using a simpler proxy model. Specifically, previous work shows that Shapley values can be computed over K-nearest neighbors (KNN) classifiers in PTIME~\cite{Jia2019-kz} and using KNN classifiers as a proxy is very effective in various real-world scenarios~\cite{Jia2021-zf}.
\end{enumerate}
In this work, we face an even harder problem given the presence of an ML pipeline $f$ in addition to the model $\mathcal{A}$. Nevertheless, as a baseline, it is important to realize that all Monte Carlo Shapley approaches~\cite{ghorbani2019data,jia2019towards} can be directly extended to support our scenario. This is because most, if not all, Monte Carlo Shapley approaches operate on {\em black-box functions} and thus, can be used directly to handle an end-to-end pipeline $\mathcal{A} \circ f$.

\inlinesection{Core Technical Problem.} Despite the existence of such a Monte Carlo baseline, there remain tremendous challenges with respect to scalability and speed --- in our experiments in \autoref{sec:evaluation}, it is not uncommon for such a Monte Carlo baseline to take a full hour to compute Shapley values even on a small dataset with only 1,000 examples. To bring data debugging and understanding into practice, we are in dire need for a more efficient and scalable alternative. Without an ML pipeline, using a KNN proxy model has been shown to be orders of magnitude faster than its Monte Carlo counterpart~\cite{Jia2019-kz} while being equally, if not more, effective on many applications~\cite{Jia2021-zf}. 

As a consequence, we focus on the following question: {\em Can we similarly use a KNN classifier as a proxy when dealing with end-to-end ML pipelines}? Today's KNN Shapley algorithm heavily relies on the structure of the KNN classifier. The presence of an ML pipeline will drastically change the underlying algorithm and time complexity --- in fact, for many ML pipelines, computation of Shapley value is \texttt{\#P}-hard even for KNN classifiers. 

%% file: approach.tex
\section{The DataScope Approach}
\label{sec:approach}

We summarize our main theoretical contribution in \autoref{sec:approach-overview}, followed by the characteristics of ML pipelines to which these results are applicable (\autoref{sec:approach-characteristics}). 
We further discuss how we can approximate many real-world pipelines as \textit{canonical pipelines} to make them compatible with our algorithmic approach (\autoref{sec:approach-approx}). 
We defer the details of our (non-trivial) theoretical results to \autoref{sec:framework}.

\subsection{Overview}
\label{sec:approach-overview}

The key technical contribution of this paper is a novel algorithmic framework that covers a large sub-family of ML pipelines whose KNN Shapley can be computed in \textsf{PTIME}. We call these pipelines \textit{canonical pipelines}.

\begin{theorem} \label{thm:shapley-using-counting-oracle}
Let $\mathcal{D}_{tr}$ be a set of $n$ training tuples, $f$ be an ML pipeline over $\mathcal{D}_{tr}$, and  $\mathcal{A}_{knn}$ be a $K$-nearest neighbor classifier. If $f$ can be expressed as an Additive Decision Diagram (ADD) with polynomial size, then computing
\begin{small}
\[
\varphi_i = \frac{1}{n} \sum_{S \subseteq \mathcal{D}_{tr} \backslash \{t_i\}} {n - 1 \choose |S|}^{-1} \left( 
    u \circ \mathcal{A}_{knn} \circ f (S \cup \{t_i\}) -
    u \circ \mathcal{A}_{knn} \circ f (S) 
\right)
\]
\end{small}
is in \textsf{PTIME} for all additive utilities $u$.
\label{theorem:main}
\end{theorem}

\begin{figure}
\centering
\includegraphics[width=0.5\textwidth]{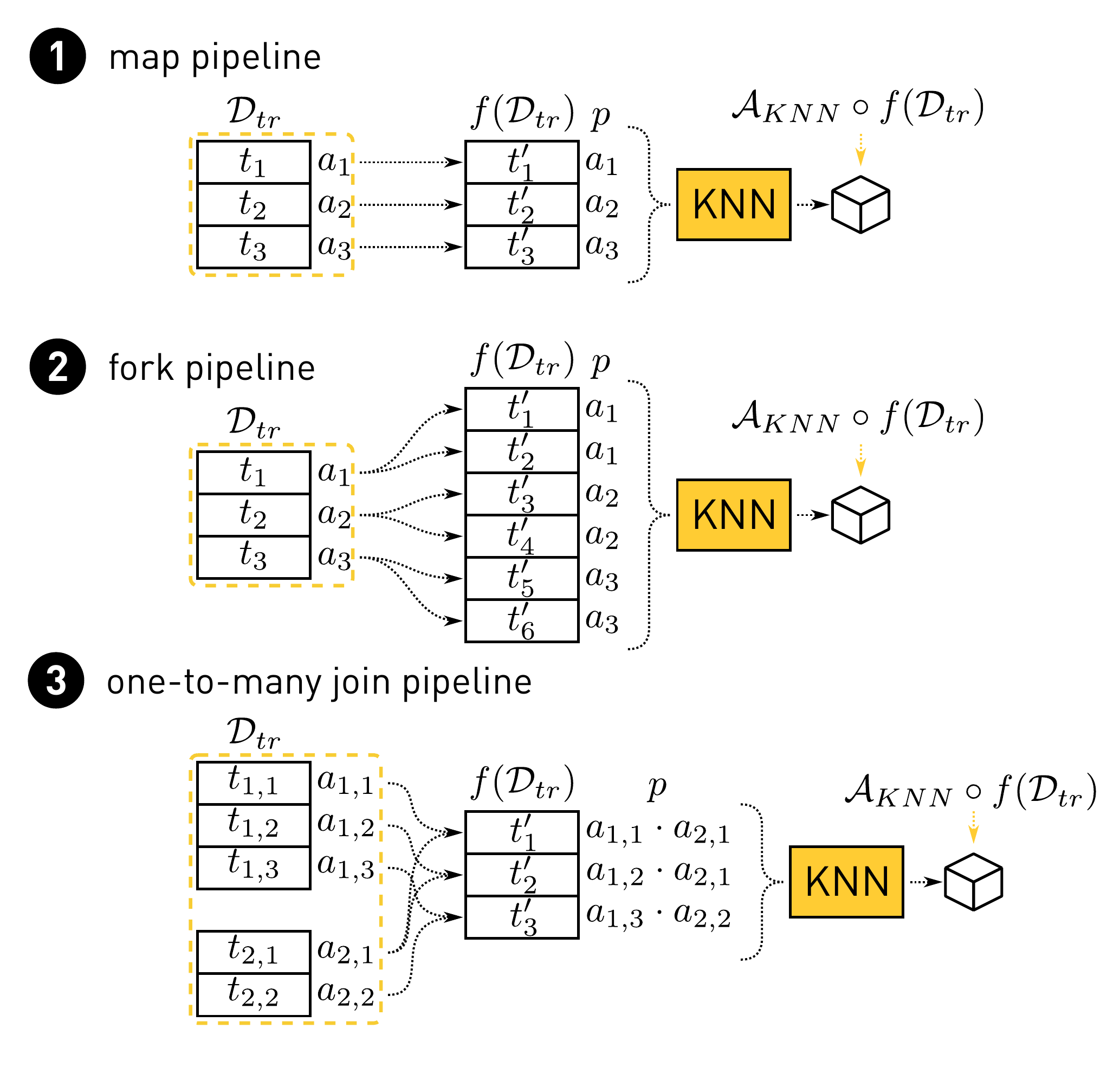}
\vspace{-2em}
\caption{Three types of canonical pipelines over which 
Shapley values can be computed in PTIME.}
\vspace{2em}
\label{fig:patterns}
\end{figure}

We leave the details about this 
Theorem to \autoref{sec:framework}.
This theorem provides a sufficient condition under which we can compute Shapley values for KNN classifiers over ML pipelines. We can instantiate this general framework with concrete types of ML pipelines.

\subsection{Canonical ML Pipelines}
\label{sec:approach-characteristics}

As a prerequisite for an efficient Shapley-value computation over pipelines, we need to understand how the removal of an input tuple $t_i$ from $\mathcal{D}_{tr}$ impacts the featurised training data $f(\mathcal{D}_{tr})$ produced by the pipeline. In particular, we need to be able to reason about the difference between $f(\mathcal{D}_{tr})$ and $f(\mathcal{D}_{tr} \setminus \{t_i\})$, which requires us to understand the \textit{data provenance}~\cite{green2007provenance,cheney2009provenance} of the pipeline $f$. In the following, we summarise three common types of pipelines (illustrated in \autoref{fig:patterns}), to which we refer as {\em canonical pipelines}. We will formally prove that Shapley values over these pipelines can be computed in PTIME in \autoref{sec:framework}. 

\inlinesection{Map pipelines} are a family of pipelines that satisfy the condition in \autoref{theorem:main}, in which the feature extraction $f$ has the following property: each input training tuple $t_i$ is transformed into a unique output training example $t_i$ with a tuple-at-a-time transformation function $h_f$: $t_i \mapsto t_i' = h_f(t_i)$. Map pipelines are the standard case for supervised learning, where each tuple of the input data is encoded as a feature vector for the model's training data. The provenance polynomial for the output $t_i'$ is $p(t_i) = a_i$ in this case, where $a_i$ denotes the presence of $t_i$ in the input to the pipeline~$f$.

\inlinesection{Fork pipelines} are a superset of Map pipelines, which requires that for each output example $t_j$, there exists a \textit{unique} 
input tuple $t_i$, such that $t_j$ is generated by applying a tuple-at-a-time transformation function $h_f$ over $t_i$: $t_j = h_f(t_i)$. As illustrated in \autoref{fig:patterns}(b), the output examples $t_1$ and $t_2$ are both generated from the input example $t_1$. Fork pipelines also satisfy the condition in \autoref{theorem:main}. Fork pipelines typically originate from data augmentation operations for supervised learning, where multiple variants of a single tuple of the input data are generated (e.g., various rotations of an image in computer vision), and each copy is encoded as a feature vector for the model's training data. The provenance polynomial for an output $t_j$ is again $p(t_j) = a_i$ in this case, where $a_i$ denotes the presence of $t_i$ in the input to the pipeline~$f$.

\inlinesection{One-to-Many Join pipelines} are a superset of Fork pipelines, which rely on the star-schema structure of the relational inputs. Given the relational inputs $\mathcal{D}_e$ (``fact table'') and $\mathcal{D}_s$  (``dimension table''), we require that, for each output example $t_k$, there exist \textit{unique} input tuples $t_i \in \mathcal{D}_e$ and $t_j \in \mathcal{D}_s$ such that $t_k$ is generated by applying a tuple-at-a-time transformation function $h_f$ over the join pair $(t_i, t_j)$: $t_k = h_f(t_i, t_j)$. One-to-Many Join pipelines also satisfy the condition in \autoref{theorem:main}. Such pipelines occur when we have multiple input datasets in supervised learning, with the ``fact'' relation holding data for the entities to classify (e.g., emails in a spam detection scenario), and the ``dimension'' relations holding additional side data for these entities, which might result in additional helpful features. 
The provenance polynomial for an output $t_k$ is $p(t_k) = a_i \cdot a_j$ in this case, where $a_i$ and $a_j$ denote the presence of $t_i$ and $t_j$ in the input to the pipeline~$f$. Note that the polynomials states that both $t_i$ and $t_j$ must be present in the input at the same time (otherwise no join pair can be formed from them).

\inlinesection{Discussion.} We note that this classification of pipelines assumes that the relational operations applied by the pipeline are restricted to the positive relational algebra (SPJU: Select, Project, Join, Union), where the pipeline applies no aggregations, and joins the input data according to the star schema. In our experience, this covers a lot of real-world use cases in modern ML infrastructures, where the ML pipeline consumes pre-aggregated input data from so-called ``feature stores,'' which is naturally modeled in a star schema.  Furthermore, pipelines in the real-world operate on relational datasets using dataframe semantics~\cite{petersohn13towards}, where unions and projections do not deduplicate their results, which (together with the absence of aggregations), has the effect that there are no additions present in provenance polynomials of the outputs of our discussed pipeline types. This pipeline model has also been proven helpful for interactive data distribution debugging~\cite{grafberger2022data,grafberger2021mlinspect}. 

\begin{figure}[t!]
\centering
\includegraphics[width=0.5\textwidth]{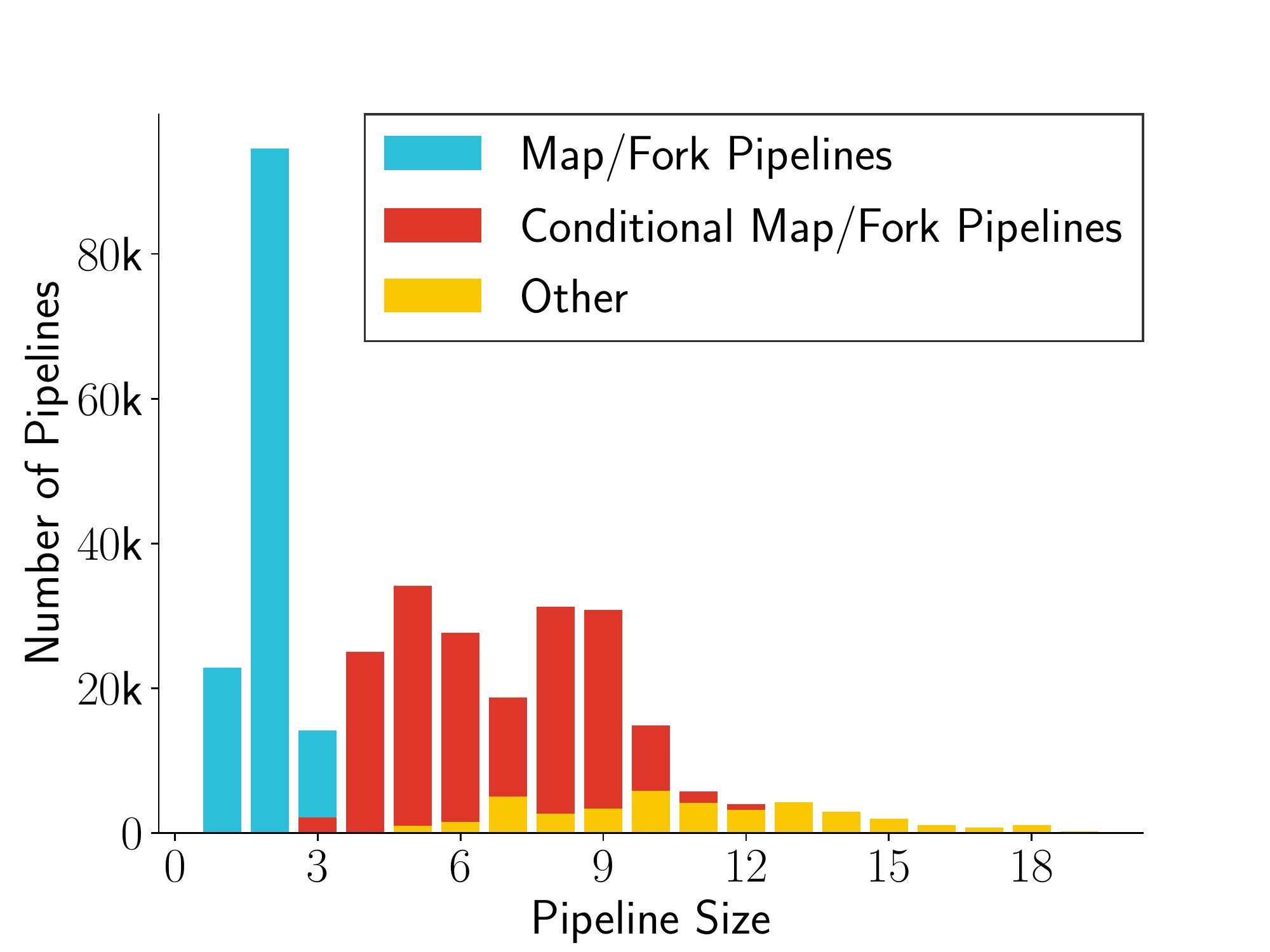}
\vspace{-1em}
\caption{A majority of real-world ML pipelines~\cite{psallidas2019data}
either already exhibit a canonical map-fork pipeline pattern, or are easily convertible to it using our approximation scheme.}
\vspace{2em}
\label{fig:coverage}
\end{figure}

\subsection{Approximating Real-World ML Pipelines}
\label{sec:approach-approx}

In practice, an ML pipeline $f$ and its corresponding ML model $\mathcal{A}$ will often not directly give us a canonical pipeline whose Shapley value can be computed in \textsf{PTIME}. The reasons for this are twofold: $(i)$~there might be no technique known to compute the Shapley value in \textsf{PTIME} for the given model; $(ii)$~the estimator/transformer operations for feature encoding in the pipeline require global aggregations (e.g., to compute the mean of an attribute for normalising it). In such cases, each output depends on the whole input, and the pipeline does not fit into one of the canonical pipeline types that we discussed earlier.

As a consequence, we \textit{approximate} an ML pipeline into a canonical pipeline in two ways.

\inlinesection{Approximating~$\mathcal{A}$.} The first approximation follows various previous efforts summarised as KNN~Shapley before, and has been shown to work well, if not better, in a diverse range of scenarios. In this step, we approximate the pipeline's ML model with a KNN classifier
\[
\mathcal{A} \mapsto \mathcal{A}_{knn}.
\]

\inlinesection{Approximating the estimator/transformer steps in $f$.} In terms of the pipeline operations, we have to deal with the global aggregations applied by the estimators for feature encoding. Common feature encoding and dimensionality reduction techniques often
base on a \texttt{reduce-map} pattern over the data:
\[
op(\mathcal{D}) = \texttt{map}(\texttt{reduce}(\mathcal{D}), \mathcal{D}).
\]

During the \texttt{reduce} step, the estimator computes some global statistics over the dataset --- e.g., the estimator for \texttt{MinMaxScaling} computes the minimum and maximum of an attribute, and the \texttt{TFIDF} estimator computes the inverse document frequencies of terms. The estimator then generates a transformer, which applies the \texttt{map} step to the data, transforming the input dataset based on the computed global statistics, e.g., to normalise each data example based on the computed minimum and maximum values in case of \texttt{MinMaxScaling} . 

The global aggregation conducted by the \texttt{reduce} step is often the key reason that we cannot compute Shapley value in \textsf{PTIME} over a given pipeline --- such a global 
aggregation requires us to enumerate
all possible subsets of
data examples, each of which 
corresponds to a potentially different
global statistic. Fortunately, we also observe, and will validate
empirically later, that the results of these global aggregations are relatively stable given different subsets of the data, especially in cases where what we want to compute is the \textit{difference} when a single example is added or removed. The approximation that we conduct is to reuse the result of the \texttt{reduce} step computed over the whole dataset $\mathcal{D}_{tr}$ for a subset $\mathcal{D} \subset \mathcal{D}_{tr}$:
\begin{equation} \label{eq:conditional-op-conversion}
    op(\mathcal{D}) = \texttt{map}(\texttt{reduce}(\mathcal{D}), \mathcal{D}) \mapsto 
    op^*(\mathcal{D}) = \texttt{map}(\texttt{reduce}(\mathcal{D}_{tr}), \mathcal{D}).
\end{equation}

In the case of scikit-learn, this means that we reuse the transformer generated by fitting the estimator on the whole dataset. Once all estimators $op$ in an input pipeline are transformed into their approximate variant $op^*$, a large majority of realistic pipelines become canonical pipelines of a \texttt{map} or \texttt{fork} pattern.

\inlinesection{Statistics of Real-world Pipelines.} A natural question is how common these families of pipelines are in practice. \autoref{fig:coverage} illustrates a case study that we conducted over 500K real-world pipelines provided by Microsoft~\cite{psallidas2019data}. We divide pipelines into three categories: (1) "pure" map/fork pipelines, based on our definition of canonical pipelines; (2) "conditional" map/fork pipelines, which are comprised of a reduce operator that can be effectively approximated using the scheme we just described; and (3) other pipelines, which contain complex operators that cannot be approximated. We see that a vast majority of pipelines we encountered in our case study fall into the first two categories that we can effectively approximate using our canonical pipelines framework.

\paragraph*{\underline{Discussion: What if These Two Approximations Fail?}}
Computing Shapley values for a generic pipeline
$(\mathcal{A}, f)$ is \#\textsf{P}-hard, and by
approximating it into
$(\mathcal{A}_{knn}, f^*)$, we obtain an 
efficient \textsf{PTIME} solution. This drastic improvement
on complexity also means that \textit{we should
expect that there exist scenarios under which
this $(\mathcal{A}, f) \mapsto (\mathcal{A}_{knn}, f^*)$
approximation is not a good approximation.}

{\em How often would this failure case happen in
practice?} 
When the training set is large, 
as illustrated in many previous studies focusing on the KNN proxy, we are confident that the $\mathcal{A} \mapsto \mathcal{A}_{knn}$ approximation should work well in many practical scenarios except those 
relying on some very strong global 
properties that KNN does not model (e.g., global population balance).
As for the $f \mapsto f^*$ approximation, we expect the failure cases to be rare, especially when the training set is large. 
In our experiments, we have empirically verified these two beliefs, which were also backed up by previous empirical results on KNN Shapley~\cite{Jia2021-zf}.

{\em What should we do when such a failure case
happens?} Nevertheless, we should expect 
such a failure case can happen in practice.
In such situations, we will resort to the Monte Carlo baseline,
which will be orders of magnitude slower but
should provide a backup alternative.
It is an interesting future direction to further 
explore the limitations of both approximations and 
develop more efficient Monte Carlo methods.

\paragraph*{\underline{Approximating Additive Utilities: Equalized Odds Difference}} 
We show how slightly more complex utilities can also be represented as additive, with a little approximation, similar to the one described above. We will demonstrate this using the ``equalized odds difference'' utility, a measure of (un)fairness commonly used in research~\cite{moritz2016equality,barocas-hardt-narayanan} that we also use in our experiments. It can be defined as such:
\begin{equation} \label{eq:eqodds-diff}
  u (\mathcal{D}_{tr}, \mathcal{D}_{val}) := \max\{ TPR_{\Delta}(\mathcal{D}_{tr}, \mathcal{D}_{val}), FPR_{\Delta}(\mathcal{D}_{tr}, \mathcal{D}_{val}) \}.
\end{equation}
Here, $TPR_{\Delta}$ and $FPR_{\Delta}$ are \emph{true positive rate difference} and \emph{false positive rate difference} respectively. We assume that each tuple $t_{tr} \in f(\mathcal{D}_{tr})$ and $t_{val} \in f(\mathcal{D}_{val})$ have some sensitive feature $g$ (e.g. ethnicity) with values taken from some finite set 
$\{G_1, G_2, ... \}$, that allows us to partition the dataset into \emph{sensitive groups}. We can define $TPR_{\Delta}$ and $FPR_{\Delta}$ respectively as
\begin{equation} \label{eq:tpr-diff}
\begin{split}
  TPR_{\Delta}(\mathcal{D}_{tr}, \mathcal{D}_{val}) &:= 
  \max_{G_i \in G} TPR_{G_i}(\mathcal{D}_{tr}, \mathcal{D}_{val})
  -
  \min_{G_j \in G} TPR_{G_j}(\mathcal{D}_{tr}, \mathcal{D}_{val}), \ \textrm{and} \\
  FPR_{\Delta}(\mathcal{D}_{tr}, \mathcal{D}_{val}) &:= 
  \max_{G_i \in G} FPR_{G_i}(\mathcal{D}_{tr}, \mathcal{D}_{val})
  -
  \min_{G_j \in G} FPR_{G_j}(\mathcal{D}_{tr}, \mathcal{D}_{val}).
\end{split}
\end{equation}
For some sensitive group $G_i$, we define $TPR_{G_i}$ and $FPR_{G_i}$ respectively as:
\begin{equation*}
\begin{split}
  TPR_{G_i}(\mathcal{D}_{tr}, \mathcal{D}_{val}) &:= \frac{\sum_{t_{val} \in f(\mathcal{D}_{val})} \mathbbm{1}\{ (\mathcal{A} \circ f (\mathcal{D}_{tr}))(t_{val}) = 1 \} \mathbbm{1} \{ y(t_{val}) = 1 \}  \mathbbm{1} \{ g(t_{val}) = G_i \} }{|\{ t_{val} \in \mathcal{D}_{val} \ : \ y(t_{val}) = 1 \wedge g(t_{val}) = G_i \}|}, \ \textrm{and} \\
  FPR_{G_i}(\mathcal{D}_{tr}, \mathcal{D}_{val}) &:= \frac{\sum_{t_{val} \in f(\mathcal{D}_{val})} \mathbbm{1}\{ (\mathcal{A} \circ f (\mathcal{D}_{tr}))(t_{val}) = 1 \} \mathbbm{1} \{ y(t_{val}) = 0 \}  \mathbbm{1} \{ g(t_{val}) = G_i \} }{|\{ t_{val} \in \mathcal{D}_{val} \ : \ y(t_{val}) = 0 \wedge g(t_{val}) = G_i \}|}
\end{split}
\end{equation*}

For a given training dataset $\mathcal{D}_{tr}$, we can determine \autoref{eq:eqodds-diff} whether $TPR_{\Delta}$ or $FPR_{\Delta}$ is going to be the dominant metric. Similarly, given that choice, we can determine a pair of sensitive groups $(G_{max}, G_{min})$ that would end up be selected as minimal and maximal in \autoref{eq:tpr-diff}. Similarly to the conversion shown in \autoref{eq:conditional-op-conversion}, we can treat these two steps as a \texttt{reduce} operation over the whole dataset. Then, if we assume that this intermediate result will remain stable over subsets of $\mathcal{D}_{tr}$, we can approximatly represent the equalized odds difference utility as an additive utility.

As an example, let us assume that we have determined that $TPR_\Delta$ dominates over $FPR_\Delta$, and similarly that the pair of sensitive groups $(G_{max}, G_{min})$ will end up being selected in \autoref{eq:tpr-diff}. Then, our tuple-wise utility $u_T$ and the scaling factor $w$ become
\begin{align*}
  u_T(y_{pred}, t_{val}) &:= TPR_{G_{max},T}(y_{pred}, t_{val}) - TPR_{G_{min},T}(y_{pred}, t_{val}), \\
  w &:= 1/|\{ t_{val} \in \mathcal{D}_{val} \ : \ y(t_{val}) = 1 \wedge g(t_{val}) = G_i \}|,
\end{align*}
where 
\begin{equation*}
  TPR_{G_i,T}(y_{pred}, t_{val}) := \mathbbm{1}\{ y_{pred} = 1 \} \mathbbm{1} \{ y(t_{val}) = 1 \}  \mathbbm{1} \{ g(t_{val}) = G_i \}.
\end{equation*}
A similar approach can be taken to define $u_T$ and $w$ for the case when $FPR_\Delta$ dominates over $TPR_\Delta$.

%% file: framework.tex

\section{Algorithm Framework: KNN Shapley Over Data Provenance}
\label{sec:framework}

We now provide details for our 
theoretical results that are 
mentioned in \autoref{sec:approach}.
We present an algorithmic framework that efficiently computes the Shapley value over the KNN accuracy utility (defined in \autoref{eq:model-acc-definition} when $\mathcal{A}$ is the KNN model).
Our framework is based on the following key ideas: (1) the computation can be reduced to computing a set of \emph{counting oracles}; (2) we can develop PTIME algorithms to compute such counting oracles for the canonical ML pipelines, by translating their \emph{provenance polynomials} into an Additive Decision Diagram (ADD).

\subsection{Counting Oracles}


We now unpack \autoref{theorem:main}.
Using the notations of data provenance introduced in \autoref{sec:prelim}, 
we can rewrite the definition of the Shapley value as follows, computing
the value of tuple $t_i$, with 
the corresponding varible $a_i \in A$:
\begin{equation} \label{eq:shap-pipeline}
\varphi_i = \frac{1}{|A|} \sum_{v \in \mathcal{V}_{A \setminus \{ a_i \}}}
\frac{u(\mathcal{D}_{tr}^f[v[a_i \gets 1]]) - u(\mathcal{D}_{tr}^f[v[a_i \gets 0]])}
{\binom{|A|-1}{|\mathrm{supp}(v)|}}.
\end{equation}
Here, $v[a_i \gets X]$ represents the same value assignment as $v$, except that we enforce $v(a_i) = X$ for some constant $X$. Moreover, the support $\mathrm{supp}(v)$ of a value assignment $v$ is the subset of variables in $A$ that are assigned value $1$ according to $v$.




\inlinesection{Nearest Neighbor Utility.}
When the downstream classifier is 
a K-nearest neighbor classifier, we have 
additional structure of 
the utility function $u(-)$ that we 
can take advantage of. 
Given a data example $t_{val}$ from the validation dataset, the hyperparameter $K$ controlling the size of the neighborhood and the set of class labels $\mathcal{Y}$, we formally define the KNN utility $u_{t_{val}, K, \mathcal{Y}}$ as follows.
Given the transformed training set
$\mathcal{D}_{tr}^f$,
let $\sigma$ be a scoring
function that computes, for each tuple
$t \in \mathcal{D}_{tr}^f$, its 
similarity with 
the validation example $t_{val}$:
$\sigma(t, t_{val})$. In the following,
we often write $\sigma(t)$ whenever
$t_{val}$ is clear from the context.
We also omit $\sigma$ when the scoring
function is clear from the context. 
Given this scoring function $\sigma$, 
the KNN utility can be defined as follows: 
\begin{equation}
u_{t_{val}, K, \mathcal{Y}} (\mathcal{D}) := 
u_T
\left( \mathrm{argmax}_{y \in \mathcal{Y}} \Big(
        \mathrm{tally}_{y, \mathrm{top}_K \mathcal{D}_{tr}^f} (\mathcal{D}_{tr}^f)
    \Big),
t_{val}
\right)
\label{eq:knn-acc-definition}
\end{equation}
where $\mathrm{top}_K \mathcal{D}_{tr}^f$ returns the tuple $t$ which ranks at the $K$-th spot 
when all tuples in $\mathcal{D}_{tr}^f$ 
are ordered by decreasing similarity $\sigma$. Given this tuple $t$ and a class label $y \in \mathcal{Y}$, the $\mathrm{tally}_{y, t}$ operator returns the number of tuples with similarity score greater or equal to $t$ that have label $y$. We assume a standard majority voting scheme where the predicted label is selected to be the one with the greatest tally ($\arg\max_y$). The accuracy is then computed by simply comparing the predicted label with the label of the validation tuple $t_{val}$.

Plugging the KNN accuracy utility into \autoref{eq:shap-pipeline}, we can augment the expression for computing $\varphi_i$ as
\begin{equation}
    \begin{split}
        \varphi_i = \frac{1}{|A|}
        \sum_{v \in \mathcal{V}_{A \setminus \{ a_i \}}}
        \sum_{\alpha=1}^{|A|}
            & \mathbbm{1} \{ \alpha = |\mathrm{supp}(v)| \}
            \binom{|A|-1}{\alpha}^{-1} \\
        \cdot \sum_{t, t' \in \mathcal{D}_{tr}^f}
            & \mathbbm{1} \{ t = \mathrm{top}_K \mathcal{D}_{tr}^f[v[a_i \gets 0]] \} \\
            \cdot & \mathbbm{1} \{ t' = \mathrm{top}_K \mathcal{D}_{tr}^f[v[a_i \gets 1]] \} \\
        \cdot \sum_{\gamma, \gamma' \in \Gamma}
            & \mathbbm{1} \{ \gamma = \mathrm{tally}_{t} \mathcal{D}_{tr}^f[v[a_i \gets 0]] \} \\
            \cdot & \mathbbm{1} \{ \gamma' = \mathrm{tally}_{t'} \mathcal{D}_{tr}^f[v[a_i \gets 1]] \} \\
            \cdot & u_{\Delta} (\gamma, \gamma')
    \end{split}
    \label{eq:shap-transform-1}
\end{equation}
where $\mathrm{tally}_{t} \mathcal{D} = (
   \mathrm{tally}_{c_1, t} \mathcal{D}
   ...
   \mathrm{tally}_{y_{|\mathcal{Y}|}, t} \mathcal{D})$
   returns a tally vector $\gamma \in \Gamma \subset \mathbb{N}^{|\mathcal{Y}|}$ consisting of the tallied occurrences of each class label $y \in \mathcal{Y}$ among tuples with similarity to $t_{val}$ greater than or equal to that of the boundary tuple $t$.
Let $\Gamma$ be all
possible tally vectors (corresponding to 
all possible label ``distributions'' over
top-$K$).

Here, the innermost utility gain function is formally defined as
$u_{\Delta} (\gamma, \gamma') := u_{\Gamma}(\gamma') - u_{\Gamma}(\gamma)$, where $u_{\Gamma}$ is defined as
\[
u_{\Gamma}(\gamma) := u_T(\mathrm{argmax}_{y \in \mathcal{Y}} \gamma, t_{val} ).
\]
Intuitively,
$u_{\Delta} (\gamma, \gamma')$
measures the utility 
difference between 
two different label distributions (i.e., tallies)
of top-$K$ examples: $\gamma$ and $\gamma'$.
$u_T(y, t_{val})$ is the tuple-wise utility for a KNN prediction (i.e., $\mathrm{argmax}_{y \in \mathcal{Y}} \gamma$) and validation tuple $t_{val}$, which is the building block of the \emph{additive utility}.
The correctness of \autoref{eq:shap-transform-1} comes from the observation that for any distinct $v \in \mathcal{V}_{A \setminus \{ a_i \}}$, there is a unique solution to all indicator functions $\mathbbm{1}$. Namely, there is a single $t$ that is the $K$-th most similar tuple when $v(a_i)=0$, and similarly, a single $t'$ when $v(a_i)=1$. Given those \emph{boundary tuples} $t$ and $t'$, the same goes for the \emph{tally vectors}:
given $\mathcal{D}_{tr}^f[v[a_i \gets 0]]$ and $\mathcal{D}_{tr}^f[v[a_i \gets 1]]$,
there exists a unique
$\gamma$ and $\gamma'$.

We can now define the following \emph{counting oracle} that computes the sum over value assignments, along with all the predicates:
\begin{equation}
    \begin{split}
        \omega_{t, t'} (\alpha, \gamma, \gamma') :=
        \sum_{v \in \mathcal{V}_{A \setminus \{ a_i \}}}
        \cdot & \mathbbm{1} \{ \alpha = |\mathrm{supp}(v)| \} \\
        \cdot & \mathbbm{1} \{ t = \mathrm{top}_K \mathcal{D}_{tr}^f[v[a_i \gets 0]] \} \\
        \cdot & \mathbbm{1} \{ t' = \mathrm{top}_K \mathcal{D}_{tr}^f[v[a_i \gets 1]] \} \\
        \cdot & \mathbbm{1} \{ \gamma = \mathrm{tally}_t \mathcal{D}_{tr}^f[v[a_i \gets 0]] \} \\
        \cdot & \mathbbm{1} \{ \gamma' = \mathrm{tally}_t \mathcal{D}_{tr}^f[v[a_i \gets 1]] \}.
    \end{split}
    \label{eq:counting-oracle}
\end{equation}

Using counting oracles, we can simplify \autoref{eq:shap-transform-1} as:
\begin{equation}
        \varphi_i = \frac{1}{N}
        \sum_{t, t' \in \mathcal{D}_{tr}^f}
        \sum_{\alpha=1}^{N}
        \binom{N-1}{\alpha}^{-1}
        \sum_{\gamma, \gamma' \in \Gamma}
        u_{\Delta} (\gamma, \gamma')
        \omega_{t, t'} (\alpha, \gamma, \gamma').
    \label{eq:shap-main}
\end{equation}

We see that the computation of $\varphi_i$ will be in \textsf{PTIME} if we can compute the counting oracles $\omega_{t, t'}$ in \textsf{PTIME} (ref. \autoref{theorem:main}).
As we will demonstrate next, this is indeed the case for the canonical pipelines that we focus on in this paper.

\subsection{Counting Oracles for Canonical Pipelines}

We start by discussing how to compute the counting oracles using ADD's in general.
We then study the canonical ML pipelines in particular and develop \textsf{PTIME} algorithms for them.


\subsubsection{Counting Oracle using ADD's} \label{sec:oracle-with-add}

We use Additive Decision Diagram (ADD) to compute the counting oracle $\omega_{t, t'}$ (\autoref{eq:counting-oracle}).
An ADD represents a Boolean function $\phi : \mathcal{V}_{A} \rightarrow \mathcal{E} \cup \{\infty\}$ that maps value assignments $v \in \mathcal{V}_{A}$ to elements of some set $\mathcal{E}$ or a special invalid element $\infty$ (see \autoref{sec:additive-decision-diagrams} for more details). 
For our purpose, we define 
$\mathcal{E} := \{1,...,|A|\} \times \Gamma \times \Gamma$, where $\Gamma$ is the set of label tally vectors. 
We then define a
function over Boolean inputs $\phi_{t, t'} : \mathcal{V}_{A}[a_i=0] \rightarrow \mathbbm{N}$ as follows:
\begin{equation} \label{eq:oracle-add-function}
    \begin{split}
        \phi_{t, t'}(v) &:= \begin{cases}
            \infty,     & \mathrm{if} \ t \not\in \mathcal{D}\big[v[a_i \gets 0]\big], \\
            \infty,     & \mathrm{if} \ t' \not\in \mathcal{D}\big[v[a_i \gets 1]\big], \\
            (\alpha, \gamma, \gamma' ),     & \mathrm{otherwise}, \\
        \end{cases} \\
        \alpha &:= |\mathrm{supp}(v)|, \\
        \gamma &:= \mathrm{tally}_{t} \mathcal{D}\big[v[a_i \gets 0]\big], \\
        \gamma' &:= \mathrm{tally}_{t'} \mathcal{D}\big[v[a_i \gets 1]\big].
    \end{split}
\end{equation}
If we can construct an ADD with a root node $n_{t, t'}$ that computes $\phi_{t, t'}(v)$, 
then the following equality holds:
\begin{equation}
    \omega_{t, t'} (\alpha, \gamma, \gamma') = \mathrm{count}_{(\alpha, \gamma, \gamma')} (n_{t, t'}).
\end{equation}
Given that 
the complexity of model counting is $O(|\mathcal{N}| \cdot |\mathcal{E}|)$ (see \autoref{eq:dd-count-recursion}) and the size of $\mathcal{E}$ is polynomial in the size of data, we have
\begin{theorem} \label{thm:decision-diagram}
If we can represent the $\phi_{t, t'}(v)$ in \autoref{eq:oracle-add-function} with an ADD of size polynomial in $|A|$ and $|\mathcal{D}_{tr}^f|$, we can compute the counting oracle $\omega_{t, t'}$ in time polynomial of $|A|$ and $|\mathcal{D}_{tr}^f|$.
\end{theorem}

\subsubsection{Constructing Polynomial-size ADD's for ML Pipelines} \label{sec:add-for-pipeline}

\begin{algorithm}[t!]
    \caption{Compiling a provenance-tracked dataset into ADD.} \label{alg:compile-dataset-to-add}
    \begin{algorithmic}[1]

    \Function{CompileADD}{}
    
    \Inputs
    \Input{$\mathcal{D}$, provenance-tracked dataset;}
    \Input{$A$, set of variables;}
    \Input{$t$, boundary tuple;}
    \Outputs
    \Output{$\mathcal{N}$, nodes of the compiled ADD;}
    
    \Begin
    
    \State $\mathcal{N} \gets \{\}$
    \State $\mathcal{P} \gets \{ (x_1, x_2) \in A \ : \exists t \in \mathcal{D}, x_1 \in p(t) \ \wedge \ x_2 \in p(t) \}$
    \State $A_{L} \gets $ \Call{GetLeafVariables}{$\mathcal{P}$} \label{alg:cmp:line:leaves}
    \For{$A_{C} \in $ \Call{GetConnectedComponents}{$\mathcal{P}$}} \label{alg:cmp:line:conn-cmp}
        \State $\mathcal{N}' \gets $ \Call{ConstructADDTree}{$A_C \setminus w_L$} \label{alg:cmp:line:add-tree}
        
        \State $A' \gets A_C \setminus w_L$
        \State $\mathcal{D}' \gets \{ t' \in \mathcal{D} \ : \ p(t') \cup A_C \neq \emptyset \}$
        \For{$v \in \mathcal{V}_{A'}$}
            \State $\mathcal{N}_C \gets $ \Call{ConstructADDChain}{$A_C \cap w_L$}
            \For{$n \in \mathcal{N}_C$}
                \State $v' \gets v \cup \{ x(n) \rightarrow 1 \}$
                \State $w_H(n) \gets |\{ t' \in \mathcal{D}' \ : \ \mathrm{eval}_{v'} p(t') = 1 \ \wedge \ \sigma(t') \geq \sigma(t) \}|$
            \EndFor
            
            \State $\mathcal{N}' \gets $ \Call{AppendToADDPath}{$\mathcal{N}'$, $\mathcal{N}_C$, $v$} \label{alg:cmp:line:append-path}
        \EndFor
        \State $\mathcal{N} \gets $ \Call{AppendToADDRoot}{$\mathcal{N}$, $\mathcal{N}'$}
        
    \EndFor
    
    \For{$x' \in p(t)$}
        \For{$n \in \mathcal{N}$ \textbf{where} $x(n) = x'$}
            \State $w_L(n) \gets \infty$
        \EndFor
    \EndFor

    \Return $\mathcal{N}$
    
    \EndFunction
    
    \end{algorithmic}
\end{algorithm}

\autoref{alg:compile-dataset-to-add} presents our main procedure \textsc{CompileADD} that constructs an ADD for a given dataset $\mathcal{D}$ made up of tuples annotated with provenance polynomials.
Invoking \textsc{CompileADD}($\mathcal{D}$, $A$, $t$) constructs an ADD with node set $\mathcal{N}$ that computes 
\begin{equation} \label{eq:oracle-add-function-single}
    \phi_{t}(v) := \begin{cases}
        \infty,     & \mathrm{if} \ t \not\in \mathcal{D}[v], \\
        \mathrm{tally}_t \mathcal{D}[v],    & \mathrm{otherwise}. \\
    \end{cases}
\end{equation}
We provide a more detailed description of \autoref{alg:compile-dataset-to-add} in \autoref{sec:apx-alg-compile-dataset-to-add-details}. 

To construct the function defined in \autoref{eq:oracle-add-function}, we need to invoke \textsc{CompileADD} once more by passing $t'$ instead of $t$ in order to obtain another diagram $\mathcal{N}'$. The final diagram is obtained by $\mathcal{N}[a_i \gets 0] + \mathcal{N}'[a_i \gets 1]$. The size of the resulting diagram will still be bounded by $O(|\mathcal{D}|)$. 

We can now examine different types of canonical pipelines and see how their structures are reflected onto the ADD's.
In summary, we can construct an ADD with polynomial size for canonical pipelines and therefore, by \autoref{thm:decision-diagram}, the computation of the corresponding counting oracles is in PTIME.

\inlinesection{One-to-Many Join Pipeline.}
In a \emph{star} database schema, this corresponds to a \emph{join} between a \emph{fact} table and a \emph{dimension} table, where each tuple from the dimension table can be joined with multiple tuples from the fact table. It can be represented by an ADD similar to the one in \autoref{fig:example-add-structure}. 

\begin{corollary} \label{col:complexity-knn-join}
For the $K$-NN accuracy utility and a one-to-many \emph{join} pipeline, which takes as input two datasets, $\mathcal{D}_F$ and $\mathcal{D}_D$, of total size $|\mathcal{D}_F| + |\mathcal{D}_D| = N$ and outputs a joined dataset of size $O(N)$, the Shapley value can be computed in $O(N^4)$ time.
\end{corollary}
We present the proof in \autoref{sec:apx-complexity-knn-join-proof} in the appendix.

\inlinesection{Fork Pipeline.}
The key characteristic of a pipeline $f$ that contains only \emph{fork} or \emph{map} operators is that the resulting dataset $f(\mathcal{D})$ has provenance polynomials with only a single variable. This is due to the absence of joins, which are the only operator that results in provenance polynomials with a combination of variables. 

\begin{corollary} \label{col:complexity-knn-fork}
For the $K$-NN accuracy utility and a \emph{fork} pipeline, which takes as input a dataset of size $N$ and outputs a dataset of size $M$, the Shapley value can be computed in $O(M^2 N^2)$ time.
\end{corollary}

We present the proof in \autoref{sec:apx-complexity-knn-fork-proof} in the appendix.

\inlinesection{Map Pipeline.}
A \emph{map} pipeline is similar to \emph{fork} pipeline in the sense that every provenance polynomial contains only a single variable. However, each variable now can appear in a provenance polynomial of \emph{at most} one tuple, in contrast to \emph{fork} pipeline where a single variable can be associated with \emph{multiple} tuples. This additional restriction results in the following corollary:

\begin{corollary} \label{col:complexity-knn-map}
For the $K$-NN accuracy utility and a \emph{map} pipeline, which takes as input a dataset of size $N$, the Shapley value can be computed in $O(N^2)$ time.
\end{corollary}

We present the proof in \autoref{sec:apx-complexity-knn-map-proof} in the appendix.

\subsection{Special Case: 1-Nearest-Neighbor Classifiers} \label{sec:special-case-1nn}

We can significantly reduce the time complexity for 1-NN classifiers, an important special case of $K$-NN classifiers that is commonly used in practice.
For each validation 
tuple 
$t_{val}$, there is always \emph{exactly} one tuple that is most similar to $t_{val}$. 
Below we illustrate how to leverage this observation to construct the counting oracle.
In the following, we 
assume that $a_i$
is the variable corresponding
to the tuple for which we hope to compute Shapley value.

Let $\phi_t$ represent the event when $t$ is the top-$1$ tuple:
\begin{equation} \label{eq:top-1-condition-map-single}
\phi_t :=
p(t) \wedge
\bigwedge_{
    \substack{
        t' \in f(\mathcal{D}_{tr}) \\
        \sigma(t') > \sigma(t)
    }
} \neg p(t').
\end{equation}
For \autoref{eq:top-1-condition-map-single} to be \emph{true} (i.e. for tuple $t$ to be the top-$1$), all tuples $t'$ where $\sigma(t') > \sigma(t)$ need to be \emph{absent} from the pipeline output. Hence, for a given value assignment $v$, all provenance polynomials that control those tuples, i.e., $p(t')$, need to evaluate to \textsf{false}.

We now construct the event
\[
\phi_{t, t'} := \phi_t[a_i/\textsf{false}] \wedge \phi_{t'}[a_i/\textsf{true}],
\]
where $\phi_t[a_i/\textsf{false}]$ means to substitue 
all appearances of $a_i$ in $\phi_t$
to \textsf{false}. This event happens 
only if if $t$ is the top-$1$ tuple 
when $a_i$ is \textsf{false} and $t'$ is the top-$1$ tuple when $a_i$ is \textsf{true}. This corresponds to the condition that our counting oracle counts models for. Expanding $\phi_{t, t'}$, we obtain
\begin{equation} \label{eq:top-1-condition-map}
\phi_{t, t'} :=
\Big(
p(t) \wedge
\bigwedge_{\substack{
    t'' \in f(\mathcal{D}_{tr}) \\
    \sigma(t'') > \sigma(t)
}} \neg p(t'')
\Big)[a_i/\textsf{false}]
\wedge
\Big(
p(t') \wedge
\bigwedge_{\substack{
    t'' \in f(\mathcal{D}_{tr}) \\
    \sigma(t'') > \sigma(t')
}} \neg p(t'')
\Big)[a_i/\textsf{true}].
\end{equation}

Note that $\phi_{t, t'}$ can only be \emph{true} if $p(t')$ is true 
when $a_i$ is \textsf{true}
and $\sigma(t) < \sigma(t')$. As a result, all provenance polynomials corresponding to tuples with a higher similarity score than that of $t$ need to evaluate to \textsf{false}. Therefore, the only polynomials that can be allowed to evaluate to \textsf{true} are those corresponding to tuples with lower similarity score than $t$. Based on these observations, we can express the counting oracle for different types of ML pipelines.

\inlinesection{Map Pipeline.}
In a \emph{map} pipeline, the provenance polynomial for each tuple $t \in f(\mathcal{D}_{tr})$ is defined by a single distinct variable $a_t \in A$. 
Furthermore, from the definition of the counting oracle (\autoref{eq:counting-oracle}), we can see that each $\omega_{t, t'}$ counts the value assignments that result in support size $\alpha$ and label tally vectors $\gamma$ and $\gamma'$.
Given our observation about the provenance polynomials that are allowed to be set to \textsf{true}, we can easily construct an expression for counting valid value assignments. Namely, we have to choose exactly $\alpha$ variables out of the set $\{t'' \in \mathcal{D} \ : \ \sigma(t'') < \sigma(t) \}$, which corresponds to tuples with lower similarity than that of $t$. This can be constructed using a \emph{binomial coefficient}. Furthermore, when $K=1$, the label tally $\gamma$ is entirely determined by the top-$1$ tuple $t$. The same observation goes for $\gamma'$ and $t'$. To denote this, we define a constant $\Gamma_L$ parameterized by some label $L$. It represents a tally vector with all values $0$ and only the value corresponding to label $L$ being set to $1$. We thus need to fix $\gamma$ to be equal to $\Gamma_{y (t)}$ (and the same for $\gamma'$). Finally, as we observed earlier, when computing $\omega_{t, t'}$ for $K=1$, the provenance polynomial of the tuple $t'$ must equal $a_i$. With these notions, we can define the counting oracle as
\begin{equation} 
\omega_{t, t'} (\alpha, \gamma, \gamma') = 
\binom{|\{t'' \in \mathcal{D} \ : \ \sigma(t'') < \sigma(t) \}|}{\alpha}
\mathbbm{1} \{ p(t')=a_i \}
\mathbbm{1} \{ \gamma = \Gamma_{y(t)} \}
\mathbbm{1} \{ \gamma' = \Gamma_{y(t')} \}.
\label{eq:oracle-1nn}
\end{equation}
Note that we always assume $\binom{a}{b}=0$ for all $a < b$. Given this,
we can prove the following corollary about \emph{map} pipelines:
\begin{corollary} \label{col:complexity-1nn-map}
For the $1$-NN accuracy utility and a \emph{map} pipeline, which takes as input a dataset of size $N$, the Shapley value can be computed in $O(N \log N)$ time.
\end{corollary}

We present the proof in \autoref{sec:apx-complexity-1nn-map-proof} in the appendix.

\inlinesection{Fork Pipeline.}
As we noted, both \emph{map} and \emph{fork} pipelines result in polynomials made up of only one variable. The difference is that in \emph{map} pipeline each variable is associated with at most one polynomial, whereas in \emph{fork} pipelines it can be associated with multiple polynomials. However, for 1-NN classifiers, this difference vanishes when it comes to Shapley value computation:
\begin{corollary} \label{col:complexity-1nn-fork}
    For the $1$-NN accuracy utility and a \emph{fork} pipeline, which takes as input a dataset of size $N$, the Shapley value can be computed in $O(N \log N)$ time.
\end{corollary}

We present the proof in \autoref{sec:apx-complexity-1nn-fork-proof} in the appendix.

%% file: evaluation.tex

\def \vucipipelines {identity,std-scaler,log-scaler,pca,mi-kmeans}
\def \vtwentynewsgroupspipelines {tf-idf,tolower-urlremove-tfidf}
\def \vfashionmnistpipelines {gauss-blur,hog-transform}

\section{Experimental Evaluation}
\label{sec:evaluation}

We evaluate the performance of 
\sysname when applied to data
debugging and repair. In this section, we present the empirical study we conducted with the goal of evaluating both quality and speed.

\subsection{Experimental Setup}

\inlinesection{Hardware and Platform.}
All experiments were conducted on Amazon AWS c5.metal instances with a 96-core Intel(R) Xeon(R) Platinum 8275CL 3.00GHz CPU and 192GB of RAM. We ran each experiment in single-thread mode.

\begin{table}[t!]
    \centering
    \scalebox{1}{
        \begin{tabular}{l|c|cc}
            \hline
            \textbf{Dataset} & \textbf{Modality} &
            \textbf{\# Examples} &
            \textbf{\# Features}
            \\
            \hline
            \at{UCIAdult}~\cite{kohavi1996scaling} & tabular & $49K$ & $14$ \\ 
            \at{Folktables}~\cite{ding2021retiring} & tabular & $1.6M$ & $10$  \\ 
            \at{FashionMNIST}~\cite{xiao2017online} & image & $14K$ & (Image) $28 \times 28$ \\ 
            \at{20NewsGroups}~\cite{joachims1996probabilistic} & text  & $1.9K$ & (Text) $20K$ after TF-IDF \\ 
           \hline \at{Higgs}~\cite{baldi2014searching} & tabular  & $11M$ & $28$ \\
            \hline
        \end{tabular}
    }
    \caption{Datasets characteristics}
    \label{tbl:datasets}
    \vspace{2em}
\end{table}

\inlinesection{Datasets.}
We assemble a collection of widely used datasets with diverse modalities (i.e. tabular, textual, and image datasets). 
\autoref{tbl:datasets} summarizes the datasets that we used.

\noindent
\textbf{(Tabular Datasets)} \at{UCI Adult} is a tabular dataset from the US census data~\cite{kohavi1996scaling}. We use the binary classification variant where the goal is to predict whether the income of a person is above or below \$50K. One of the features is `sex,' which we use as a \emph{sensitive attribute} to measure group fairness with respect to male and female subgroups. A very similar dataset is \at{Folktables}, which was developed to redesign and extend the original \at{UCI Adult} dataset with various aspects interesting to the fairness community~\cite{ding2021retiring}. We use the `income' variant of this dataset, which also has a `sex' feature and has a binary label corresponding to the \$50K income threshold. Another tabular dataset that we use for large-scale experiments is the \at{Higgs} dataset, which has $28$ features that represent physical properties of particles in an accelerator~\cite{baldi2014searching}. The goal is to predict whether the observed signal produces Higgs bosons or not.

\noindent
\textbf{(Non-tabular Datasets)}
We used two non-tabular datasets. One is \at{FashionMNIST}, which contains $28\times 28$ grayscale images of 10 different categories of fashion items~\cite{xiao2017online}. To construct a binary classification task, we take only images of the classes `shirt' and `T-shirt.'
We also use \at{TwentyNewsGroups}, which is a dataset with text obtained from newsgroup posts categorized into 20 topics~\cite{joachims1996probabilistic}. To construct a binary classification task, we take only two newsgroup categories, `\at{sci.med}' and `\at{comp.graphics}.' The task is to predict the correct category for a given piece of text.

\inlinesection{Feature Processing Pipelines.}
We obtained a dataset with about $500K$ machine learning workflow instances from internal Microsoft users~\cite{psallidas2019data}. Each workflow consists of a dataset, a feature extraction pipeline, and an ML model. We identified a handful of the most representative pipelines and translated them to \texttt{sklearn} pipelines. We list the pipelines used in our experiments in \autoref{tbl:pipelines}.


\begin{table}[t!]
    \centering
    \scalebox{1}{
        \begin{tabular}{l|ccc}
            \hline
            \thead{Pipeline} & \thead{Dataset \\ Modality} & \thead{w/ \\ Reduce} & \thead{Operators} \\
            \hline
            \at{Identity} & tabular & false & $\emptyset$ \\ 
            \at{Standard Scaler} & tabular & true & $\mathtt{StandardScaler}$ \\ 
            \at{Logarithmic Scaler} & tabular & true & $\mathtt{Log1P} \circ \mathtt{StandardScaler}$ \\ 
            \at{PCA} & tabular & true & $\mathtt{PCA}$ \\ 
            \at{Missing Indicator + KMeans} & tabular & true & $\mathtt{MissingIndicator} \oplus \mathtt{KMeans}$ \\ 
            \hline
            \at{Gaussian Blur} & image & false & $\mathtt{GaussBlur}$  \\ 
            \at{Histogram of Oriented Gradients} & image & false & $\mathtt{HogTransform}$  \\ 
            \hline
            \at{TFIDF} & text & true & $\mathtt{CountVectorizer} \circ \mathtt{TfidfTransformer}$  \\
            \multirow{2}{*}{\at{Tolower + URLRemove + TFIDF}} & 
            \multirow{2}{*}{text} & 
            \multirow{2}{*}{false} &
            $\begin{array}{cc}
                 & \mathtt{TextToLower} \circ \mathtt{UrlRemover} \\
                 & \circ \mathtt{CountVectorizer} \circ \mathtt{TfidfTransformer}
            \end{array}$ \\
            \hline
        \end{tabular}
    }
    \caption{Feature extraction pipelines used in experiments.}
    \label{tbl:pipelines}
    \vspace{2em}
\end{table}

As \autoref{tbl:pipelines} shows, we used pipelines of varying complexity. The data modality column indicates which types of datasets we applied each pipeline to. Some pipelines are pure map pipelines, while some implicitly require a reduce operation. \autoref{tbl:pipelines} shows the operators contained by each pipeline. They are combined either using a composition symbol $\circ$, i.e., operators are applied in sequence; or a concatenation symbol $\oplus$, i.e., operators are applied in parallel and their output vectors are concatenated. Some operators are taken directly from \texttt{sklearn} ($\mathtt{StandardScaler}$, $\mathtt{PCA}$, $\mathtt{MissingIndicator}$, $\mathtt{KMeans}$, $\mathtt{CountVectorizer}$, and $\mathtt{TfidfTransformer}$), while others require customized implementations: (1) $\mathtt{Log1P}$, using the \texttt{log1p} function from \texttt{numpy}; (2) $\mathtt{GaussBlur}$, using the \texttt{gaussian\_filter} function from \texttt{scipy}; (3) $\mathtt{HogTransform}$, using the 
\texttt{hog} function from \texttt{skimage}; (4) $\mathtt{TextToLower}$, using the built-in \texttt{tolower} Python function; and (5) $\mathtt{UrlRemover}$, using a simple regular expression.

\underline{\emph{Fork Variants}:}
We also create a ``fork'' version of the above pipelines, by prepending each with a $\mathtt{DataProvider}$ operator.
It simulates distinct data providers that each provides a portion of the data. The original dataset is split into a given number of groups (we set this number to $100$ in our experiments). 
We compute importance for each group, and we conduct data repairs on entire groups all at once.

\inlinesection{Models.}
We use three machine learning models as the downstream
ML model following the previous feature extraction
pipelines: \at{XGBoost}, \at{Logistic Regression}, and \at{KNearest Neighbor}. We use the \texttt{LogisticRegression} and \texttt{KNeighborsClassifier} provided by the \texttt{sklearn} package. We use the default hyper-parameter values except that we set \texttt{max\_iter} to 5,000 for \at{Logistic Regression} and \texttt{n\_neighbors} to $1$ for the \at{KNearest Neighbor}.

\inlinesection{Data Debugging Methods.}
We apply different data debugging methods and compare them based on their effect on model quality and the computation time that they require:

\begin{itemize}[leftmargin=*]
    \item \underline{\at{Random}} --- We measure importance with a random number and thus apply data repairs in random order. 
    \item \underline{\at{TMC Shapley x10} and \at{TMC Shapley x100}} --- We express importance as Shapley values computed using the Truncated Monte-Carlo (TMC) method~\cite{ghorbani2019data}, with 10 and 100 Monte-Carlo iterations, respectively. We then follow the computed importance in ascending order to repair data examples.
    \item \underline{\at{DataScope}} --- This is our $K$-nearest-neighbor based method for efficiently computing the Shapley value. We then follow the computed importance in ascending order to repair data examples.
    \item \underline{\at{DataScope Interactive}} --- While the above methods compute importance scores only once at the beginning of the repair, the speed of \at{DataScope} allows us to \textit{recompute} the importance after \emph{each} data repair. We call this strategy \at{DataScope Interactive}.
\end{itemize}

\inlinesection{Protocol.}
In most of our experiments (unless explicitly stated otherwise), we simulate importance-driven data repair scenarios performed on a given \emph{training dataset}. In each experimental run, we select a dataset, pipeline, model, and a data repair method. 
We compute the importance using the utility 
defined over a \emph{validation set}.
Training data repairs are conducted one unit at a time until all units are examined. The order of units is determined by the specific repair method. We divide the range between $0\%$ data examined and $100\%$ data examined into $100$ checkpoints. At each checkpoint we measure the quality of the given model on a separate \emph{test dataset} using some metric (e.g. accuracy). For importance-based repair methods, we also measure the time spent on computing importance. 
We repeat each experiment $10$ times and report the median as well as the $90$-th percentile range (either shaded or with error bars).

\subsection{Results}

Following the protocol of ~\cite{Li2021-sg,Jia2021-zf}, we 
start by flipping certain amount of labels in the training dataset.
We then use a given data debugging method to go through the dataset and repair labels by replacing each label with the correct one. 
As we progress through the dataset, we measure the model quality on a separate test dataset using a metric such as accuracy or equalized odds difference (a commonly used fairness metric). Our goal is to achieve the best possible quality while at the same time having to examine the least possible amount of data. 
Depending on whether the 
pipeline is an original one or its fork variant, we have slightly different approaches to label corruption and repair. For original pipelines, each label can be flipped with some probability (by default this is $50\%$). Importance is computed for independent data examples, and repairs are performed independently as well.
For fork variants, data examples are divided into groups corresponding to their respective data providers. By default, we set the number of data providers to $100$. Each label inside a single group is flipped based on a fixed probability. However, this probability differs across data providers (going from $0\%$ to $100\%$). Importance is computed for individual providers, and when a provider is selected for repair, all its labels get repaired.

\begin{figure*}
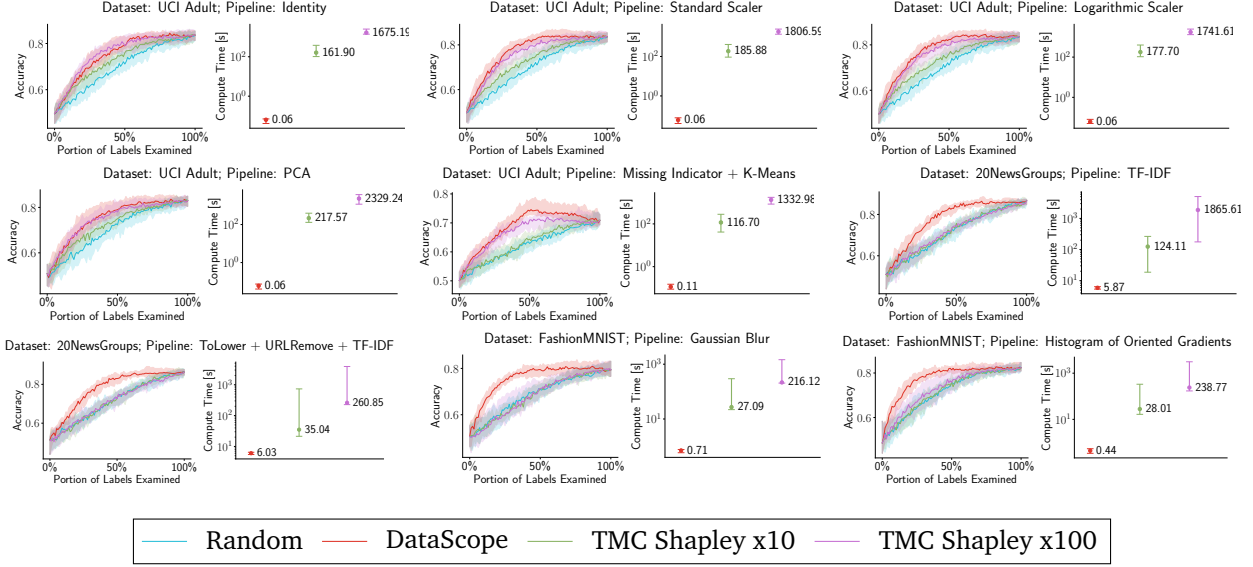

    \centering

    \def \vscenario {label-repair}
    \def \vtrainsize {1k}
    \def \vrepairgoal {accuracy}
    \def \vproviders {0}
    \def \vxpipetype {map}
    \def \vmodel {xgb}
    \def \vxmodel {XGBoost}
    
    \makeatletter
    \def \vdataset {UCI}
    \@for\vpipeline:={\vucipipelines}\do{
        \begin{subfigure}[b]{.32\linewidth}
            \def \vpath {figures/scenario=\vscenario/trainsize=\vtrainsize/repairgoal=\vrepairgoal/providers=\vproviders/model=\vmodel/dataset=\vdataset/pipeline=\vpipeline/report.figure.pdf}
            \includegraphics[width=\linewidth]{\vpath}
        \end{subfigure}
    }
    \def \vdataset {TwentyNewsGroups}
    \@for\vpipeline:={\vtwentynewsgroupspipelines}\do{
        \begin{subfigure}[b]{.32\linewidth}
            \def \vpath {figures/scenario=\vscenario/trainsize=\vtrainsize/repairgoal=\vrepairgoal/providers=\vproviders/model=\vmodel/dataset=\vdataset/pipeline=\vpipeline/report.figure.pdf}
            \includegraphics[width=\linewidth]{\vpath}
        \end{subfigure}
    }
    \def \vdataset {FashionMNIST}
    \@for\vpipeline:={\vfashionmnistpipelines}\do{
        \begin{subfigure}[b]{.32\linewidth}
            \def \vpath {figures/scenario=\vscenario/trainsize=\vtrainsize/repairgoal=\vrepairgoal/providers=\vproviders/model=\vmodel/dataset=\vdataset/pipeline=\vpipeline/report.figure.pdf}
            \includegraphics[width=\linewidth]{\vpath}
        \end{subfigure}
    }
    \makeatother

    \begin{subfigure}[c][][c]{\linewidth}
        \begin{center}
            \vspace{10pt}
            \begin{tikzpicture}
                \begin{customlegend}[legend columns=-1,legend style={column sep=5pt}]
                    \addlegendimage{myblue}\addlegendentry{Random}
                    \addlegendimage{myred}\addlegendentry{DataScope}
                    
                    \addlegendimage{mygreen}\addlegendentry{TMC Shapley x10}
                    \addlegendimage{mypink}\addlegendentry{TMC Shapley x100}
                \end{customlegend}
            \end{tikzpicture}
        \end{center}
    \end{subfigure}

    \caption{Label Repair experiment results over various combinations of datasets (\vtrainsize\xspace samples) and \vxpipetype\xspace pipelines. We optimize for \vrepairgoal\xspace. The model is \vxmodel\xspace.}
    \label{fig:exp-\vscenario-\vrepairgoal-\vxpipetype-\vmodel-\vtrainsize}
    \vspace{2em}
\end{figure*}

\inlinesection{Improving Accuracy with Label Repair.}
In this set of experiments, we aim to improve the accuracy as much as possible with as least as possible labels examined. 
We show the case for XGBoost in \autoref{fig:exp-label-repair-accuracy-map-xgb-1k}
while leave other scenarios (logistic regression, $K$-nearest neighbor, original pipelines and fork variants) to Appendix (\autoref{fig:exp-label-repair-accuracy-map-logreg-1k} to \autoref{fig:exp-label-repair-accuracy-fork-xgb-1k}. Figures differ with respect to the target model used (XGBoost, logistic regression, or $K$-nearest neighbor) and the type of pipeline (either map or fork). In each figure, we show results for different pairs of dataset and pipeline, and we measure the performance of the target model as well as the time it took to compute the importance scores.

We see that \sysname is significantly faster than 
TMC-based methods. The speed-up is in the order of $100\times$ to $1,000\times$ for models such as logistic regression. For models requiring slightly longer training time (e.g., XGBoost), the speed-up can be up to $10,000\times$.

In terms of quality, we see that \sysname is comparable with or better than the TMC-based methods (mostly for the logistic regression model),
both outperforming the \at{Random} repair method. In certain cases, \sysname,
despite its orders of magnitude speed-up,
also clearly dominates the TMC-based methods, especially
when the pipelines produce features of high-dimensional datasets (such as the text-based pipelines used for the \at{20NewsGroups} dataset and the image-based pipelines used for the \at{FashionMNIST} dataset).

\begin{figure*}
    \centering

    \def \vscenario {label-repair}
    \def \vtrainsize {1k}
    \def \vrepairgoal {fairness}
    \def \vproviders {0}
    \def \vdirtybias {0.0}
    \def \vxpipetype {map}
    \def \vmodel {xgb}
    \def \vxmodel {XGBoost}
    
    \makeatletter
    \def \vdataset {FolkUCI}
    \@for\vpipeline:={\vucipipelines}\do{
        \@for\vutility:={acc,eqodds}\do{
            \begin{subfigure}[b]{.49\linewidth}
                \def \vpath {figures/scenario=\vscenario/trainsize=\vtrainsize/repairgoal=\vrepairgoal/providers=\vproviders/model=\vmodel/dirtybias=\vdirtybias/dataset=\vdataset/pipeline=\vpipeline/utility=\vutility/report.figure.pdf}
                \includegraphics[width=\linewidth]{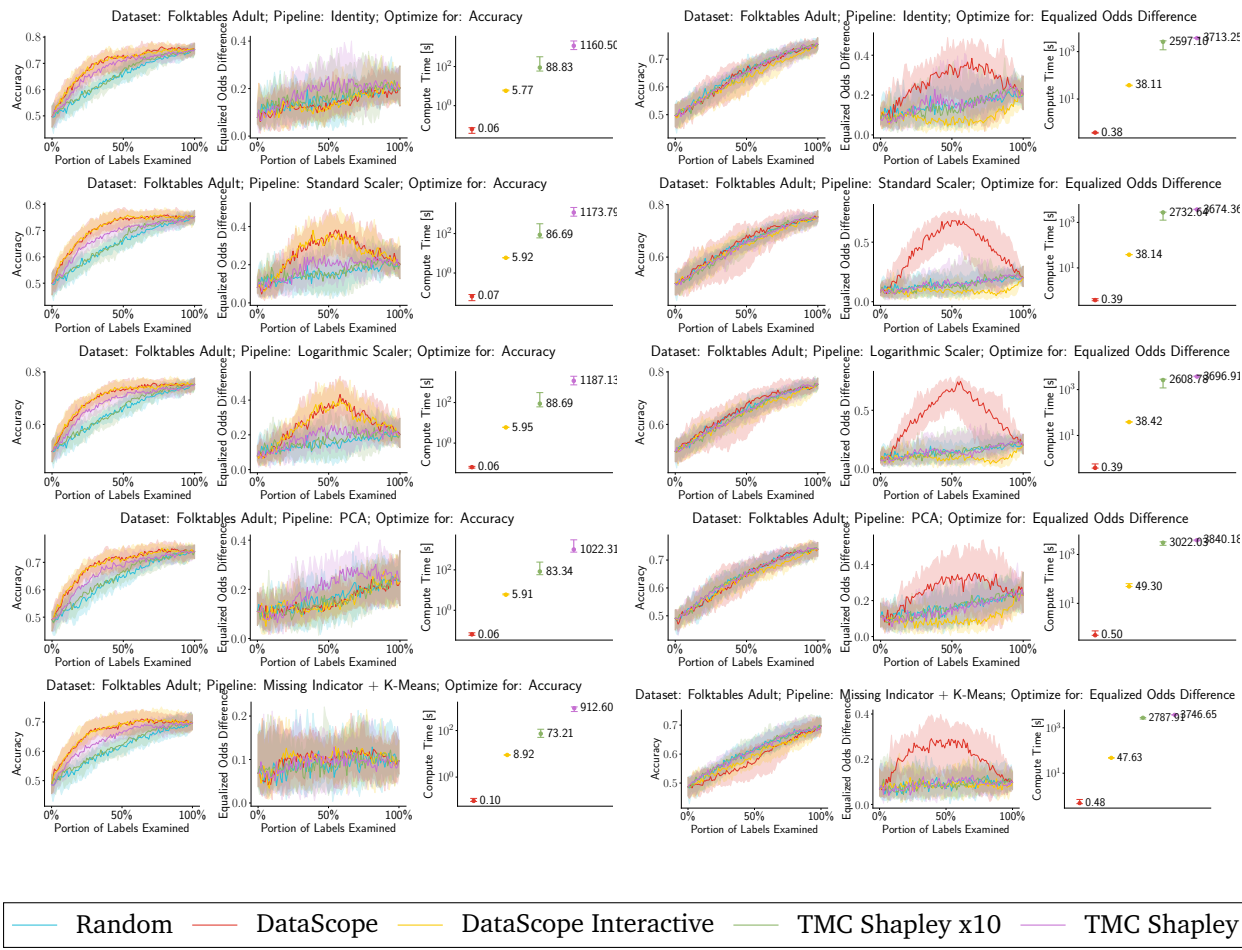}
            \end{subfigure}
        }
    }
    \makeatother

    \begin{subfigure}[c][][c]{\linewidth}
        \begin{center}
            \vspace{10pt}
            \begin{tikzpicture}
                \begin{customlegend}[legend columns=-1,legend style={column sep=5pt}]
                    \addlegendimage{myblue}\addlegendentry{Random}
                    \addlegendimage{myred}\addlegendentry{DataScope}
                    \addlegendimage{myyellow}\addlegendentry{DataScope Interactive}
                    \addlegendimage{mygreen}\addlegendentry{TMC Shapley x10}
                    \addlegendimage{mypink}\addlegendentry{TMC Shapley x100}
                \end{customlegend}
            \end{tikzpicture}
        \end{center}
    \end{subfigure}

    \caption{Label Repair experiment results over various combinations of datasets (\vtrainsize\xspace samples) and \vxpipetype\xspace pipelines. We optimize for \vrepairgoal\xspace. The model is \vxmodel\xspace.}
    \label{fig:exp-\vscenario-\vrepairgoal-\vxpipetype-\vmodel-\vtrainsize}
    \vspace{2em}
\end{figure*}

\inlinesection{Improving Accuracy and Fairness.}
We then explore the relationship between accuracy and fairness when performing label repairs. 
\autoref{fig:exp-label-repair-fairness-map-xgb-1k}
shows the result for XGBoost over original pipelines
and we leave other scenarios to the Appendix (\autoref{fig:exp-label-repair-fairness-map-logreg-1k} to \autoref{fig:exp-label-repair-fairness-fork-xgb-1k}). In these experiments we only use the two tabular datasets \at{UCI Adult} and \at{Folktables}, which have a `sex' feature that we use to compute group fairness using \emph{equalized odds difference}, one of the commonly used fairness metrics~\cite{moritz2016equality}.
We use equalized odds difference as 
the utility function for both \sysname and TMC-based methods.

We first see that being able to debug
specifically for fairness is important --- the left panel of \autoref{fig:exp-label-repair-fairness-map-xgb-1k} illustrates the behavior of 
optimizing for accuracy whereas the right 
panel illustrates the behavior of 
optimizing for fairness. 
In this example, the 100\% clean
dataset is unfair.
When optimizing
for accuracy, we see that the unfairness
of model can also increase. On the other hand, when taking fairness into consideration, \at{DataScope Interactive} is
able to maintain fairness while 
improving accuracy significantly ---
the unfairness increase of 
\at{DataScope Interactive} only 
happens at the very end of the cleaning
process, where all other ``fair''
data examples have already been cleaned.  This is likely due to the way that the equalized odds difference utility is approximated in \at{DataScope}. When computing the utility, we first make a choice on which $G_i$ and $G_j$ to choose in \autoref{eq:tpr-diff}, as well as a choice between $TPR_\Delta$ and $FPR_\Delta$ in \autoref{eq:eqodds-diff}; only then we compute the Shapley value. We assume that these choices are stable over the entire process of label repair. However, if these choices are ought to change, only \at{DataScopeInteractive} is able to make the necessary adjustment because the Shapley value is recomputed after every repair. 

In terms of speed, \sysname significantly outperforms TMC-based methods --- in the order of $100\times$ to $1,000\times$ for models like logistic regression and up to $10,000\times$ for XGBoost.
In terms of quality, \sysname is comparable to TMC-based methods, while 
\at{DataScope Interactive}, in certain cases, dramatically outperforms \sysname and TMC-based methods.
\at{DataScope Interactive} achieves much better fairness (measured by equalized odds difference, lower the better) 
while maintaining similar, if not better, accuracy
compared with other methods.
When optimizing for fairness we can observe that sometimes
non-interactive methods suffer in pipelines that use standard scalers. 
It might be possible that importance scores do not remain stable over the course of our data repair process. Because equalized odds difference is a non-trivial measure, even though it may work in the beginning of our process, it might mislead us in the wrong direction after some portion of the labels get repaired. As a result, being able to compute 
data importance frequently, which is enabled by 
our efficient algorithm, is crucial to effectively 
navigate and balance accuracy and fairness.

\inlinesection{Scalability.} We now evaluate the quality and speed of \sysname for larger training datasets. We test the runtime for various sizes of the training set ($10k$-$1M$), the validation set ($100$-$10k$), and the number of features ($100$-$1k$). As expected, the impact of training set size and validation set size is roughly linear. Furthermore, we see that even for large datasets, \sysname can compute Shapley scores in minutes.

When integrated into an intearctive data repair workflow, this could have a dramatic impact on the productivity of data scientists. We have clearly observed that Monte Carlo approaches do improve the efficiency of importance-based data debugging. However, given their lengthy runtime, one could argue that many users would likely prefer not to wait and consequently end up opting for the random repair approach. What \sysname offers is a viable alternative to random which is equally attainable while at the same time offering the significant efficiency gains provided by Shapley-based importance.

\begin{figure}[t!]
    \centering
    \begin{subfigure}[b]{.30\linewidth}
    \includegraphics[width=\linewidth]{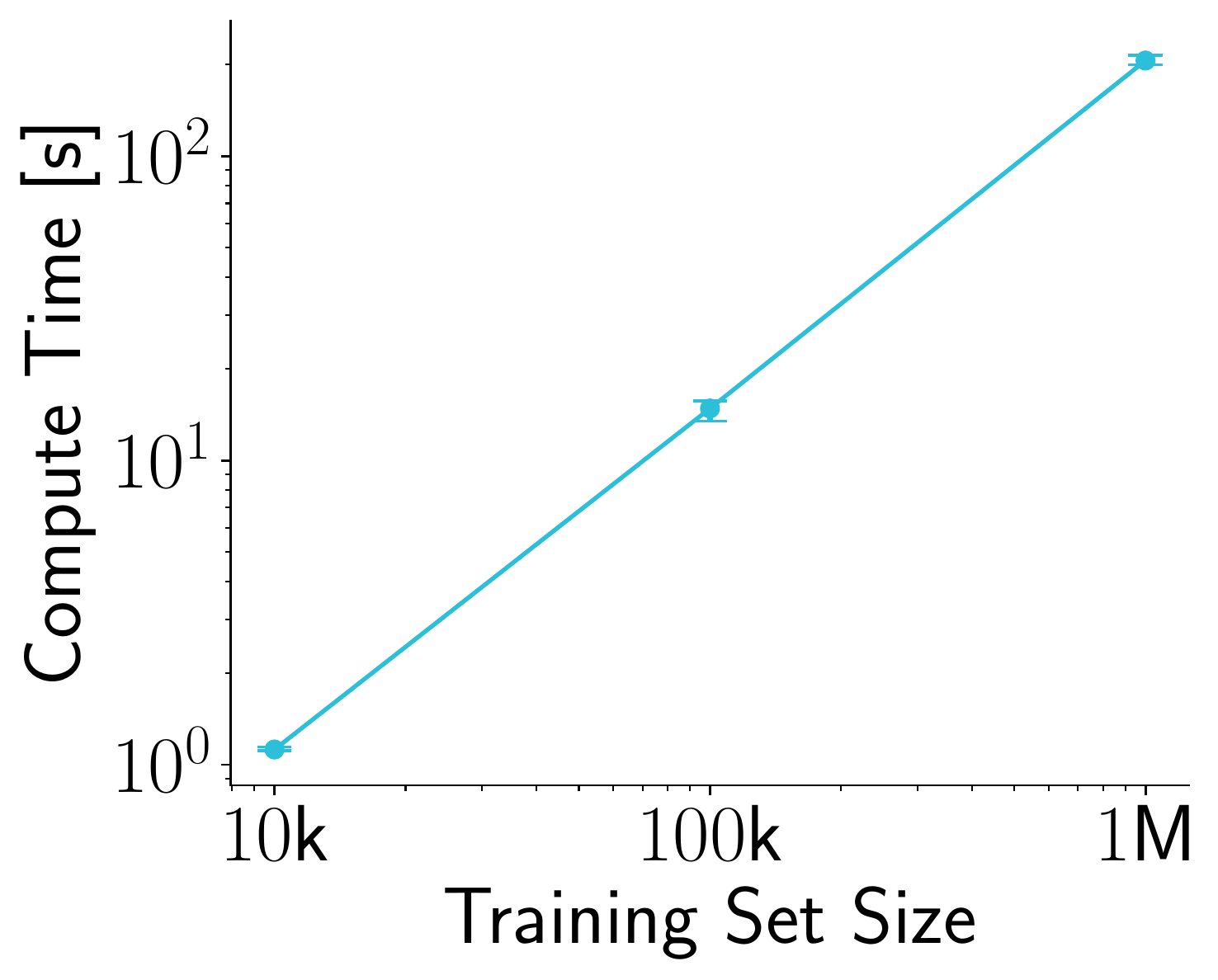}
    \end{subfigure}
    \begin{subfigure}[b]{.30\linewidth}
        \includegraphics[width=\linewidth]{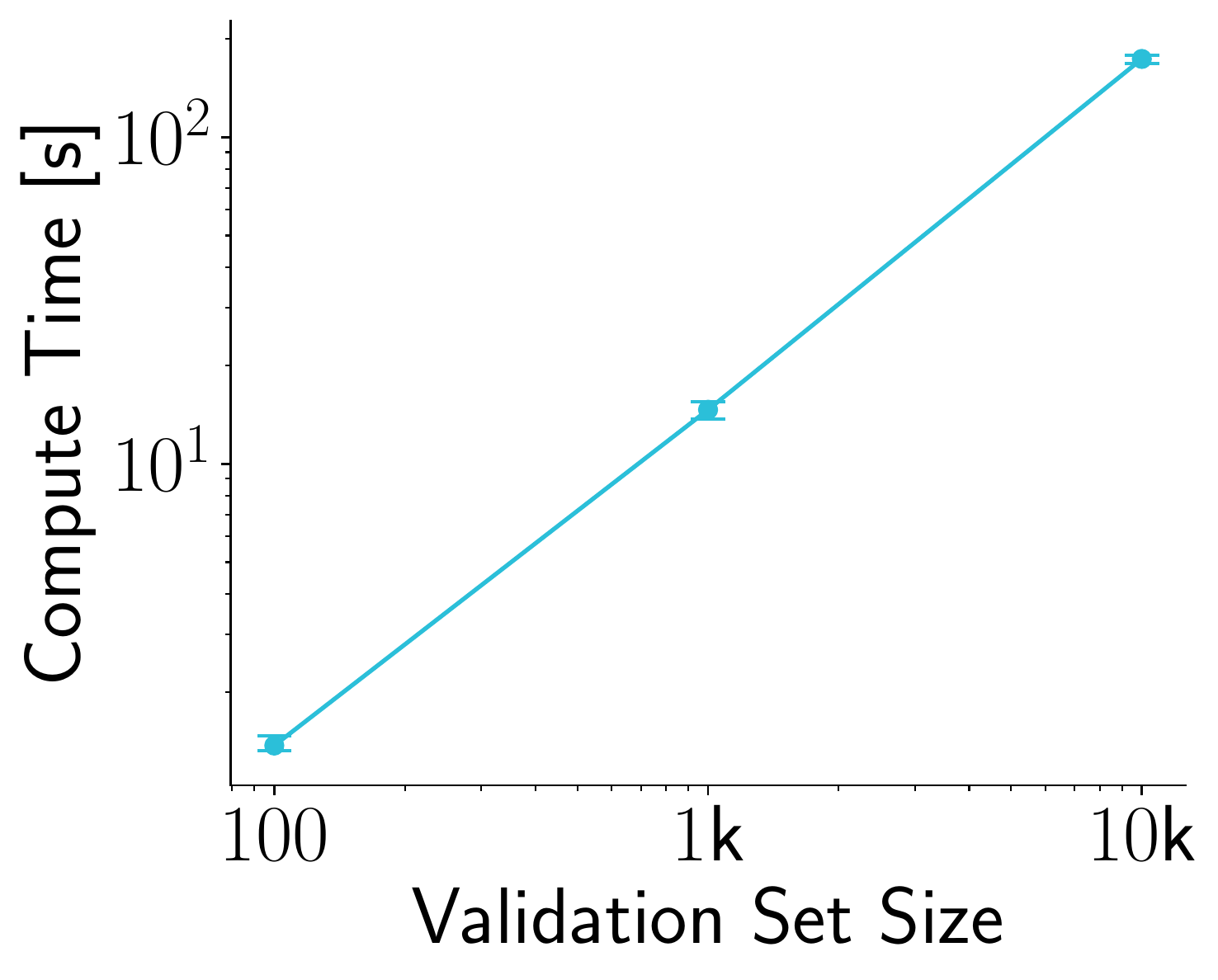}
    \end{subfigure}
    \begin{subfigure}[b]{.30\linewidth}
        \includegraphics[width=\linewidth]{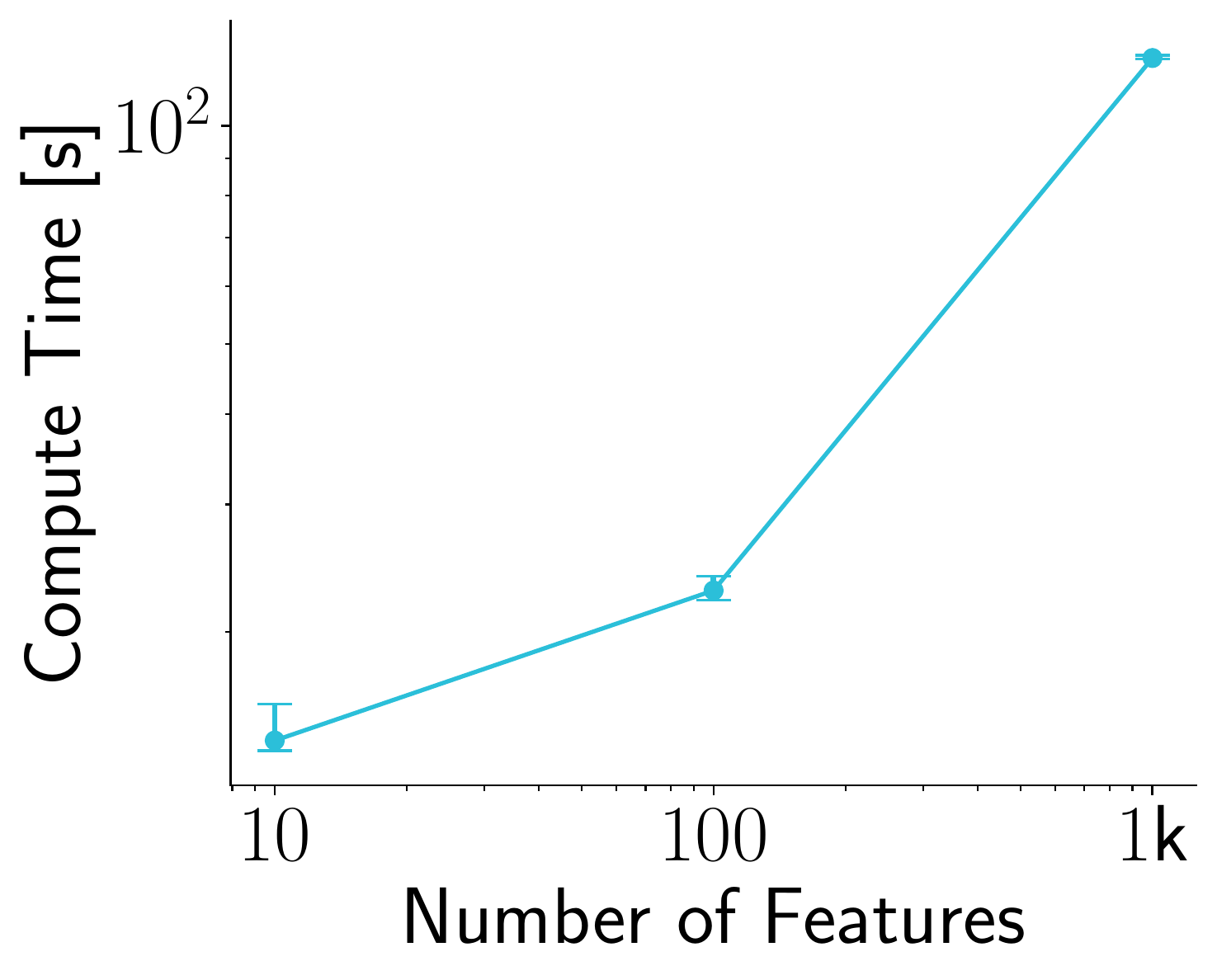}
    \end{subfigure}

    \caption{Scalability of \sysname}
    \label{fig:exp-compute-time}
    \vspace{2em}
\end{figure}

%% file: related.tex
\section{Related Work}

\inlinesection{Analyzing ML models.}
Understanding model predictions and handling problems when they arise has been an important area since the early days. In recent years, this topic has gained even more traction and is better known under the terms explainability and interpretability~\cite{adadi2018peeking,guidotti2018survey,gilpin2018explaining}. Typically, the goal is to understand why a model is making some specific prediction for a data example. Some approaches use surrogate models to produce explanations~\cite{ribeiro2016should, krishnan2017palm}. In computer vision, saliency maps have gained prominence~\cite{zeiler2014visualizing, shrikumar2017learning}. Saliency maps can be seen as a type of a \emph{feature importance} approach to explainability, although they also aim at interpreting the internals of a model. Other feature importance frameworks have also been developed~\cite{sundararajan2017axiomatic,covert2021explaining}, and some even focus on Shapley-based feature importance~\cite{lundberg2017unified}.

Another approach to model interpretability can be referred to as \emph{data importance} (i.e. using training data examples to explain predictions). This can have broader applications, including data valuation~\cite{pei2020datapricing}. One important line of work expresses data importance in terms of \emph{influence fucnctions}~\cite{basu2020second, koh2019accuracy, koh2017understanding, sharchilev2018finding}. Another line of work expresses data importance using Shapley values. Some apply Monte Carlo methods for efficiently computing it~\cite{ghorbani2019data} and some take advantage of the KNN model~\cite{jia2019towards,Jia2019-kz}. The KNN model has also been used for computing an entropy-based importance method targeted specifically for the data-cleaning application~\cite{karlas2020nearest}. In general, all of the mentioned approaches focus on interpreting a single model, without the ability to efficiently analyze it in the context of an end-to-end ML pipeline.

\inlinesection{Analyzing Data Processing Pipelines.}
For decades, the data management community has been studying how to analyze data importance for data-processing pipelines through various forms of query analysis. Some broader approaches include: causal analysis of interventions to queries~\cite{meliou2014causality}, and reverse data management~\cite{meliou2011reverse}. Methods also differ in terms of what is the target of their explanation, that is, with respect to what is the query output being explained. Some methods target queried relations~\cite{kanagal2011sensitivity}. Others target specific predicates that make up a query~\cite{roy2015explaining, roy2014formal}. Finally, some methods target specific tuples in input relations~\cite{meliou2012tiresias}.

A prominent line of work employs \emph{data provenance} as a means to analyze data processing pipelines~\cite{buneman2001and}. \emph{Provenance semiring} represents a theoretical framework for dealing with provenance~\cite{green2007provenance}. This framework gives us as a theoretical foundation to develop algorithms for analyzing ML pipelines. This analysis typically requires us to operate in an exponentially large problem space. This can be made manageable by transforming provenance information to decision diagrams through a process known as \emph{knowledge compilation}~\cite{Jha2011}. However, this framework is not guaranteed to lead us to tractable solutions. Some types of queries have been shown to be quite difficult~\cite{amsterdamer2011provenance, amsterdamer2011limitations}. In this work, we demonstrate tractability of a concrete class of pipelines compried of both a set of feature extraction operators, as well as an ML model.

Recent research under the umbrella of ``mlinspect''~ \cite{grafberger2022data,grafberger2021mlinspect,schelter2022screening} details how to compute data provenance over end-to-end ML pipelines similar to the ones in the focus of this work, based on lightweight code instrumentation.

\inlinesection{Joint Analysis of End-to-End Pipelines.}
Joint analysis of machine learning pipelines is a relatively novel, but nevertheless, an important field~\cite{polyzotis2017data}. Systems such as Data X-Ray can debug data processing pipelines by finding groups of data errors that might have the same cause~\cite{wang2015xray}. Some work has been done in the area of end-to-end pipeline compilation to tensor operations~\cite{nakandala2020tensor}. A notable piece of work leverages influence functions as a method for analyzing pipelines comprising of a model and a post-processing query~\cite{wu2020complaint}. This work also leverages data provenance as a key ingredient. In general, there are indications that data provenance is going to be a key ingredient of future ML systems~\cite{agrawal2020cloudy,schelter2022screening}, something that our own system depends upon.

%% file: conclusion.tex
\section{Conclusion}

We present
\texttt{Ease.ML/DataScope}, the first system that efficiently computes Shapley values of training examples over an \emph{end-to-end} ML pipeline.
Our core contribution is a 
novel algorithmic framework  that computes Shapley value over a specific family of ML pipelines that we call \textit{canonical pipelines}, connecting decades of research on relational data provenance and recent advancement of machine earning. 
For many subfamilies of canonical pipelines, computing Shapley value is in \textsf{PTIME}, contrasting the exponential complexity of computing Shapley value in general.
We conduct extensive experiments illustrating different use cases and utilities.
Our results show that \texttt{DataScope} is up to four orders of magnitude faster over state-of-the-art Monte Carlo-based methods, while being comparably, and often even more, effective in data debugging.

%% file: implementation.tex
\section{Implementation}

We integrate\footnote{\url{https://github.com/schelterlabs/arguseyes/blob/datascope/arguseyes/refinements/_data_valuation.py}} our Shapley value computation approach into the {\em ArgusEyes}~\cite{schelter2022screening} platform. ArgusEyes leverages the {\em mlinspect}~\cite{grafberger2022data,grafberger2021mlinspect} library to execute and instrument Python-based ML pipelines that combine code from popular data science libraries such as pandas, scikit-learn and keras. During execution, mlinspect extracts a ``logical query plan'' of the pipeline operations, modeling them as a dataflow computation with relational operations (e.g., originating from pandas) and ML-specific operations (e.g., feature encoders) which are treated as extended projections. Furthermore, ArgusEyes captures and materialises the relational inputs of the pipeline (e.g., CSV files read via pandas) and the numerical outputs of the pipeline (e.g., the labels and feature matrix for the training data). 

The existing abstractions in mlinspect map well to our pipeline model from \autoref{sec:problem-formal}, where we postulate that a pipeline maps a set of relational inputs $\mathcal{D}_{tr} = \{ \mathcal{D}_{e}, \mathcal{D}_{s_1},\dots,\mathcal{D}_{s_k} \}$ to vectorised labeled training examples $\{z_i = (x_i, y_i)\}$ for a subsequent ML model. Furthermore, mlinspect has built-in support for computing polynomials for the why-provenance~\cite{green2007provenance} of its outputs. This provenance information allows us to match a pipeline to its canonical counterpart as defined in \autoref{sec:approach-characteristics} and apply our techniques from \autoref{sec:framework} to compute Shapley values for the input data.

\autoref{lst:arguseyes} depicts a simplified implementation of data valuation for map pipelines in ArgusEyes. As discussed, the ArgusEyes platform executes the pipeline and extracts and materialises its relational \texttt{inputs}, its numerical \texttt{outputs} and the corresponding \texttt{provenance} polynomials. First, we compute Shapley values for the labeled rows $\{z_i = (x_i, y_i)\}$ of the training feature matrix produced by the pipeline, based on our previously published efficient KNN Shapley method~\cite{Jia2019-kz} (lines~10-13). Next, we retrieve the materialised relational input \texttt{fact\_table} for $\mathcal{D}_e$ (the ``fact table'' in cases of multiple inputs in a star schema), as well as the provenance polynomials \texttt{provenance\_fact\_table} for $\mathcal{D}_e$ and \texttt{provenance\_X\_train} for the training samples $\{z_i\}$~(lines~15-17). Finally, we annotate the rows of the \texttt{fact\_table} with a new column \texttt{shapley\_value} where we store the computed Shapley value for each input row. We assign the values by matching the provenance polynomials of $\mathcal{D}_e$ and $\{z_i = (x_i, y_i)\}$~(lines~19~24).

\begin{Python}[frame=none,captionpos=b,texcl=true,numbers=left,xleftmargin=.25in,caption={Simplified implementation of data valuation for map pipelines in ArgusEyes.},label={lst:arguseyes}]
class DataValuation():
 def compute(self, inputs, outputs, provenance):
  # Access captured pipeline outputs 
  # (train/test features and labels)
  X_train = outputs[Output.X_TRAIN]
  X_test = outputs[Output.X_TEST]
  y_train = outputs[Output.Y_TRAIN]
  y_test = outputs[Output.Y_TEST]
  # Compute Shapley values via KNN approximation
  shapley_values = self._shapley_knn(
    X_train, y_train, self.k,
    X_test[:self.num_test_samples, :], 
    y_test[:self.num_test_samples, :])
  # Input data and provenance
  fact_table = inputs[Input.FACT_TABLE]
  provenance_fact_table = provenance[Input.FACT_TABLE]      	
  provenance_X_train = provenance[Output.X_TRAIN]    
  # Annotate input tuples with their Shapley values	
  for polynomial, shapley_value in 
    zip(provenance_X_train, shapley_values):
    for entry in polynomial:
      if entry.input_id == fact_table.input_id:
    	row = provenance_fact_table.row_id(entry.tuple_id)
        fact_table.at[row, 'shapley_value'] = shapley_value         
\end{Python}              

We provide an executable end-to-end example for data valuation over complex pipelines in the form of a jupyter notebook at \textcolor{blue}{\url{https://github.com/schelterlabs/arguseyes/blob/datascope/arguseyes/example_pipelines/demo-shapley-pipeline.ipynb}}, which shows how to leverage ArgusEyes to identify mislabeled samples in a computer vision pipeline\footnote{\url{https://github.com/schelterlabs/arguseyes/blob/datascope/arguseyes/example_pipelines/product-images.py}} on images of fashion products.

%% file: evaluation-appendix.tex
\section{Additional Experiments}

\autoref{fig:exp-label-repair-accuracy-map-logreg-1k} to \autoref{fig:exp-label-repair-accuracy-fork-xgb-1k};
\autoref{fig:exp-label-repair-fairness-map-logreg-1k} to \autoref{fig:exp-label-repair-fairness-fork-xgb-1k}

\begin{figure*}
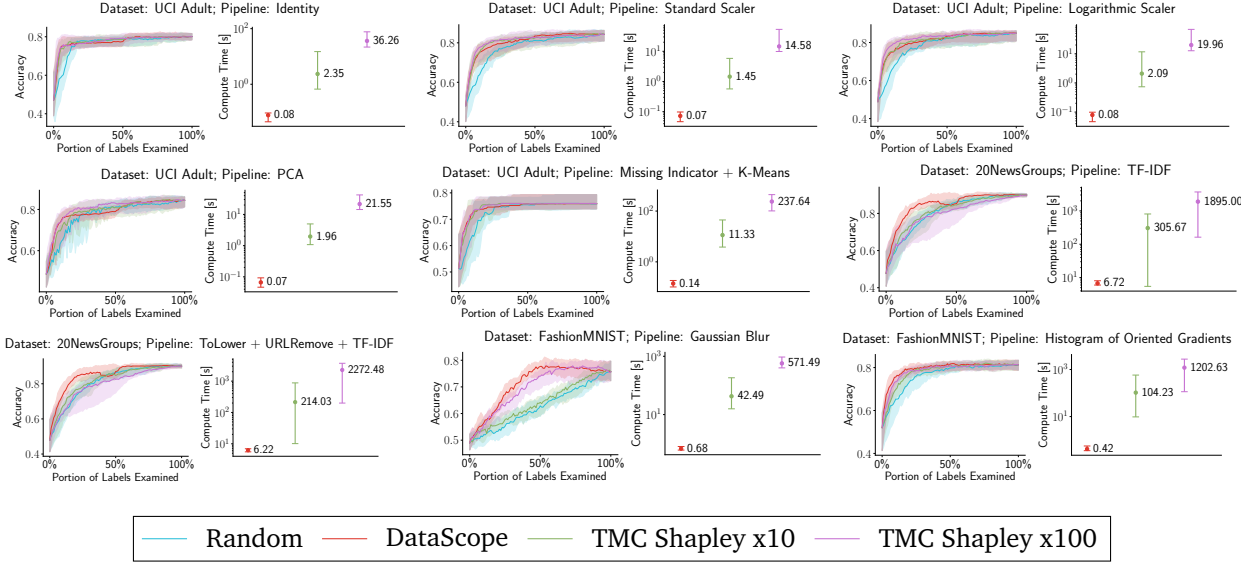

    \centering

    \def \vscenario {label-repair}
    \def \vtrainsize {1k}
    \def \vrepairgoal {accuracy}
    \def \vproviders {0}
    \def \vxpipetype {map}
    \def \vmodel {logreg}
    \def \vxmodel {logistic regression}
    
    \makeatletter
    \def \vdataset {UCI}
    \@for\vpipeline:={\vucipipelines}\do{
        \begin{subfigure}[b]{.32\linewidth}
            \def \vpath {figures/scenario=\vscenario/trainsize=\vtrainsize/repairgoal=\vrepairgoal/providers=\vproviders/model=\vmodel/dataset=\vdataset/pipeline=\vpipeline/report.figure.pdf}
            \includegraphics[width=\linewidth]{\vpath}
        \end{subfigure}
    }
    \def \vdataset {TwentyNewsGroups}
    \@for\vpipeline:={\vtwentynewsgroupspipelines}\do{
        \begin{subfigure}[b]{.32\linewidth}
            \def \vpath {figures/scenario=\vscenario/trainsize=\vtrainsize/repairgoal=\vrepairgoal/providers=\vproviders/model=\vmodel/dataset=\vdataset/pipeline=\vpipeline/report.figure.pdf}
            \includegraphics[width=\linewidth]{\vpath}
        \end{subfigure}
    }
    \def \vdataset {FashionMNIST}
    \@for\vpipeline:={\vfashionmnistpipelines}\do{
        \begin{subfigure}[b]{.32\linewidth}
            \def \vpath {figures/scenario=\vscenario/trainsize=\vtrainsize/repairgoal=\vrepairgoal/providers=\vproviders/model=\vmodel/dataset=\vdataset/pipeline=\vpipeline/report.figure.pdf}
            \includegraphics[width=\linewidth]{\vpath}
        \end{subfigure}
    }
    \makeatother

    \begin{subfigure}[c][][c]{\linewidth}
        \begin{center}
            \vspace{10pt}
            \begin{tikzpicture}
                \begin{customlegend}[legend columns=-1,legend style={column sep=5pt}]
                    \addlegendimage{myblue}\addlegendentry{Random}
                    \addlegendimage{myred}\addlegendentry{DataScope}
                    
                    \addlegendimage{mygreen}\addlegendentry{TMC Shapley x10}
                    \addlegendimage{mypink}\addlegendentry{TMC Shapley x100}
                \end{customlegend}
            \end{tikzpicture}
        \end{center}
    \end{subfigure}

    \caption{Label Repair experiment results over various combinations of datasets (\vtrainsize\xspace samples) and \vxpipetype\xspace pipelines. We optimize for \vrepairgoal\xspace. The model is \vxmodel\xspace.}
    \label{fig:exp-\vscenario-\vrepairgoal-\vxpipetype-\vmodel-\vtrainsize}
\end{figure*}

\begin{figure*}
    \centering

    \def \vscenario {label-repair}
    \def \vtrainsize {1k}
    \def \vrepairgoal {accuracy}
    \def \vproviders {0}
    \def \vxpipetype {map}
    \def \vmodel {knn}
    \def \vxmodel {K-nearest neighbor}
    
    \makeatletter
    \def \vdataset {UCI}
    \@for\vpipeline:={\vucipipelines}\do{
        \begin{subfigure}[b]{.32\linewidth}
            \def \vpath {figures/scenario=\vscenario/trainsize=\vtrainsize/repairgoal=\vrepairgoal/providers=\vproviders/model=\vmodel/dataset=\vdataset/pipeline=\vpipeline/report.figure.pdf}
            \includegraphics[width=\linewidth]{\vpath}
        \end{subfigure}
    }
    \def \vdataset {TwentyNewsGroups}
    \@for\vpipeline:={\vtwentynewsgroupspipelines}\do{
        \begin{subfigure}[b]{.32\linewidth}
            \def \vpath {figures/scenario=\vscenario/trainsize=\vtrainsize/repairgoal=\vrepairgoal/providers=\vproviders/model=\vmodel/dataset=\vdataset/pipeline=\vpipeline/report.figure.pdf}
            \includegraphics[width=\linewidth]{\vpath}
        \end{subfigure}
    }
    \def \vdataset {FashionMNIST}
    \@for\vpipeline:={\vfashionmnistpipelines}\do{
        \begin{subfigure}[b]{.32\linewidth}
            \def \vpath {figures/scenario=\vscenario/trainsize=\vtrainsize/repairgoal=\vrepairgoal/providers=\vproviders/model=\vmodel/dataset=\vdataset/pipeline=\vpipeline/report.figure.pdf}
            \includegraphics[width=\linewidth]{\vpath}
        \end{subfigure}
    }
    \makeatother

    \begin{subfigure}[c][][c]{\linewidth}
        \begin{center}
            \vspace{10pt}
            \begin{tikzpicture}
                \begin{customlegend}[legend columns=-1,legend style={column sep=5pt}]
                    \addlegendimage{myblue}\addlegendentry{Random}
                    \addlegendimage{myred}\addlegendentry{DataScope}
                    
                    \addlegendimage{mygreen}\addlegendentry{TMC Shapley x10}
                    \addlegendimage{mypink}\addlegendentry{TMC Shapley x100}
                \end{customlegend}
            \end{tikzpicture}
        \end{center}
    \end{subfigure}

    \caption{Label Repair experiment results over various combinations of datasets (\vtrainsize\xspace samples) and \vxpipetype\xspace pipelines. We optimize for \vrepairgoal\xspace. The model is \vxmodel\xspace.}
    \label{fig:exp-\vscenario-\vrepairgoal-\vxpipetype-\vmodel-\vtrainsize}
\end{figure*}

\begin{figure*}
    \centering

    \def \vscenario {label-repair}
    \def \vtrainsize {1k}
    \def \vrepairgoal {accuracy}
    \def \vproviders {100}
    \def \vxpipetype {fork}
    \def \vmodel {logreg}
    \def \vxmodel {logistic regression}
    
    \makeatletter
    \def \vdataset {UCI}
    \@for\vpipeline:={\vucipipelines}\do{
        \begin{subfigure}[b]{.32\linewidth}
            \def \vpath {figures/scenario=\vscenario/trainsize=\vtrainsize/repairgoal=\vrepairgoal/providers=\vproviders/model=\vmodel/dataset=\vdataset/pipeline=\vpipeline/report.figure.pdf}
            \includegraphics[width=\linewidth]{\vpath}
        \end{subfigure}
    }
    \def \vdataset {TwentyNewsGroups}
    \@for\vpipeline:={\vtwentynewsgroupspipelines}\do{
        \begin{subfigure}[b]{.32\linewidth}
            \def \vpath {figures/scenario=\vscenario/trainsize=\vtrainsize/repairgoal=\vrepairgoal/providers=\vproviders/model=\vmodel/dataset=\vdataset/pipeline=\vpipeline/report.figure.pdf}
            \includegraphics[width=\linewidth]{\vpath}
        \end{subfigure}
    }
    \def \vdataset {FashionMNIST}
    \@for\vpipeline:={\vfashionmnistpipelines}\do{
        \begin{subfigure}[b]{.32\linewidth}
            \def \vpath {figures/scenario=\vscenario/trainsize=\vtrainsize/repairgoal=\vrepairgoal/providers=\vproviders/model=\vmodel/dataset=\vdataset/pipeline=\vpipeline/report.figure.pdf}
            \includegraphics[width=\linewidth]{\vpath}
        \end{subfigure}
    }
    \makeatother

    \begin{subfigure}[c][][c]{\linewidth}
        \begin{center}
            \vspace{10pt}
            \begin{tikzpicture}
                \begin{customlegend}[legend columns=-1,legend style={column sep=5pt}]
                    \addlegendimage{myblue}\addlegendentry{Random}
                    \addlegendimage{myred}\addlegendentry{DataScope}
                    
                    \addlegendimage{mygreen}\addlegendentry{TMC Shapley x10}
                    \addlegendimage{mypink}\addlegendentry{TMC Shapley x100}
                \end{customlegend}
            \end{tikzpicture}
        \end{center}
    \end{subfigure}

    \caption{Label Repair experiment results over various combinations of datasets (\vtrainsize\xspace samples) and \vxpipetype\xspace pipelines. We optimize for \vrepairgoal\xspace. The model is \vxmodel\xspace.}
    \label{fig:exp-\vscenario-\vrepairgoal-\vxpipetype-\vmodel-\vtrainsize}
\end{figure*}

\begin{figure*}
    \centering

    \def \vscenario {label-repair}
    \def \vtrainsize {1k}
    \def \vrepairgoal {accuracy}
    \def \vproviders {100}
    \def \vxpipetype {fork}
    \def \vmodel {knn}
    \def \vxmodel {K-nearest neighbor}
    
    \makeatletter
    \def \vdataset {UCI}
    \@for\vpipeline:={\vucipipelines}\do{
        \begin{subfigure}[b]{.32\linewidth}
            \def \vpath {figures/scenario=\vscenario/trainsize=\vtrainsize/repairgoal=\vrepairgoal/providers=\vproviders/model=\vmodel/dataset=\vdataset/pipeline=\vpipeline/report.figure.pdf}
            \includegraphics[width=\linewidth]{\vpath}
        \end{subfigure}
    }
    \def \vdataset {TwentyNewsGroups}
    \@for\vpipeline:={\vtwentynewsgroupspipelines}\do{
        \begin{subfigure}[b]{.32\linewidth}
            \def \vpath {figures/scenario=\vscenario/trainsize=\vtrainsize/repairgoal=\vrepairgoal/providers=\vproviders/model=\vmodel/dataset=\vdataset/pipeline=\vpipeline/report.figure.pdf}
            \includegraphics[width=\linewidth]{\vpath}
        \end{subfigure}
    }
    \def \vdataset {FashionMNIST}
    \@for\vpipeline:={\vfashionmnistpipelines}\do{
        \begin{subfigure}[b]{.32\linewidth}
            \def \vpath {figures/scenario=\vscenario/trainsize=\vtrainsize/repairgoal=\vrepairgoal/providers=\vproviders/model=\vmodel/dataset=\vdataset/pipeline=\vpipeline/report.figure.pdf}
            \includegraphics[width=\linewidth]{\vpath}
        \end{subfigure}
    }
    \makeatother

    \begin{subfigure}[c][][c]{\linewidth}
        \begin{center}
            \vspace{10pt}
            \begin{tikzpicture}
                \begin{customlegend}[legend columns=-1,legend style={column sep=5pt}]
                    \addlegendimage{myblue}\addlegendentry{Random}
                    \addlegendimage{myred}\addlegendentry{DataScope}
                    
                    \addlegendimage{mygreen}\addlegendentry{TMC Shapley x10}
                    \addlegendimage{mypink}\addlegendentry{TMC Shapley x100}
                \end{customlegend}
            \end{tikzpicture}
        \end{center}
    \end{subfigure}

    \caption{Label Repair experiment results over various combinations of datasets (\vtrainsize\xspace samples) and \vxpipetype\xspace pipelines. We optimize for \vrepairgoal\xspace. The model is \vxmodel\xspace.}
    \label{fig:exp-\vscenario-\vrepairgoal-\vxpipetype-\vmodel-\vtrainsize}
\end{figure*}

\begin{figure*}
    \centering

    \def \vscenario {label-repair}
    \def \vtrainsize {1k}
    \def \vrepairgoal {accuracy}
    \def \vproviders {100}
    \def \vxpipetype {fork}
    \def \vmodel {xgb}
    \def \vxmodel {XGBoost}
    
    \makeatletter
    \def \vdataset {UCI}
    \@for\vpipeline:={\vucipipelines}\do{
        \begin{subfigure}[b]{.32\linewidth}
            \def \vpath {figures/scenario=\vscenario/trainsize=\vtrainsize/repairgoal=\vrepairgoal/providers=\vproviders/model=\vmodel/dataset=\vdataset/pipeline=\vpipeline/report.figure.pdf}
            \includegraphics[width=\linewidth]{\vpath}
        \end{subfigure}
    }
    \def \vdataset {TwentyNewsGroups}
    \@for\vpipeline:={\vtwentynewsgroupspipelines}\do{
        \begin{subfigure}[b]{.32\linewidth}
            \def \vpath {figures/scenario=\vscenario/trainsize=\vtrainsize/repairgoal=\vrepairgoal/providers=\vproviders/model=\vmodel/dataset=\vdataset/pipeline=\vpipeline/report.figure.pdf}
            \includegraphics[width=\linewidth]{\vpath}
        \end{subfigure}
    }
    \def \vdataset {FashionMNIST}
    \@for\vpipeline:={\vfashionmnistpipelines}\do{
        \begin{subfigure}[b]{.32\linewidth}
            \def \vpath {figures/scenario=\vscenario/trainsize=\vtrainsize/repairgoal=\vrepairgoal/providers=\vproviders/model=\vmodel/dataset=\vdataset/pipeline=\vpipeline/report.figure.pdf}
            \includegraphics[width=\linewidth]{\vpath}
        \end{subfigure}
    }
    \makeatother

    \begin{subfigure}[c][][c]{\linewidth}
        \begin{center}
            \vspace{10pt}
            \begin{tikzpicture}
                \begin{customlegend}[legend columns=-1,legend style={column sep=5pt}]
                    \addlegendimage{myblue}\addlegendentry{Random}
                    \addlegendimage{myred}\addlegendentry{DataScope}
                    
                    \addlegendimage{mygreen}\addlegendentry{TMC Shapley x10}
                    \addlegendimage{mypink}\addlegendentry{TMC Shapley x100}
                \end{customlegend}
            \end{tikzpicture}
        \end{center}
    \end{subfigure}

    \caption{Label Repair experiment results over various combinations of datasets (\vtrainsize\xspace samples) and \vxpipetype\xspace pipelines. We optimize for \vrepairgoal\xspace. The model is \vxmodel\xspace.}
    \label{fig:exp-\vscenario-\vrepairgoal-\vxpipetype-\vmodel-\vtrainsize}
\end{figure*}

\begin{figure*}
    \centering

    \def \vscenario {label-repair}
    \def \vtrainsize {1k}
    \def \vrepairgoal {fairness}
    \def \vproviders {0}
    \def \vdirtybias {0.0}
    \def \vxpipetype {map}
    \def \vmodel {logreg}
    \def \vxmodel {logistic regression}
    
    \makeatletter
    \def \vdataset {FolkUCI}
    \@for\vpipeline:={\vucipipelines}\do{
        \@for\vutility:={acc,eqodds}\do{
            \begin{subfigure}[b]{.49\linewidth}
                \def \vpath {figures/scenario=\vscenario/trainsize=\vtrainsize/repairgoal=\vrepairgoal/providers=\vproviders/model=\vmodel/dirtybias=\vdirtybias/dataset=\vdataset/pipeline=\vpipeline/utility=\vutility/report.figure.pdf}
                \includegraphics[width=\linewidth]{\vpath}
            \end{subfigure}
        }
    }
    \makeatother

    \begin{subfigure}[c][][c]{\linewidth}
        \begin{center}
            \vspace{10pt}
            \begin{tikzpicture}
                \begin{customlegend}[legend columns=-1,legend style={column sep=5pt}]
                    \addlegendimage{myblue}\addlegendentry{Random}
                    \addlegendimage{myred}\addlegendentry{DataScope}
                    \addlegendimage{myyellow}\addlegendentry{DataScope Interactive}
                    \addlegendimage{mygreen}\addlegendentry{TMC Shapley x10}
                    \addlegendimage{mypink}\addlegendentry{TMC Shapley x100}
                \end{customlegend}
            \end{tikzpicture}
        \end{center}
    \end{subfigure}

    \caption{Label Repair experiment results over various combinations of datasets (\vtrainsize\xspace samples) and \vxpipetype\xspace pipelines. We optimize for \vrepairgoal\xspace. The model is \vxmodel\xspace.}
    \label{fig:exp-\vscenario-\vrepairgoal-\vxpipetype-\vmodel-\vtrainsize}
\end{figure*}

\begin{figure*}
    \centering

    \def \vscenario {label-repair}
    \def \vtrainsize {1k}
    \def \vrepairgoal {fairness}
    \def \vproviders {0}
    \def \vdirtybias {0.0}
    \def \vxpipetype {map}
    \def \vmodel {knn}
    \def \vxmodel {K-nearest neighbor}
    
    \makeatletter
    \def \vdataset {FolkUCI}
    \@for\vpipeline:={\vucipipelines}\do{
        \@for\vutility:={acc,eqodds}\do{
            \begin{subfigure}[b]{.49\linewidth}
                \def \vpath {figures/scenario=\vscenario/trainsize=\vtrainsize/repairgoal=\vrepairgoal/providers=\vproviders/model=\vmodel/dirtybias=\vdirtybias/dataset=\vdataset/pipeline=\vpipeline/utility=\vutility/report.figure.pdf}
                \includegraphics[width=\linewidth]{\vpath}
            \end{subfigure}
        }
    }
    \makeatother

    \begin{subfigure}[c][][c]{\linewidth}
        \begin{center}
            \vspace{10pt}
            \begin{tikzpicture}
                \begin{customlegend}[legend columns=-1,legend style={column sep=5pt}]
                    \addlegendimage{myblue}\addlegendentry{Random}
                    \addlegendimage{myred}\addlegendentry{DataScope}
                    \addlegendimage{myyellow}\addlegendentry{DataScope Interactive}
                    \addlegendimage{mygreen}\addlegendentry{TMC Shapley x10}
                    \addlegendimage{mypink}\addlegendentry{TMC Shapley x100}
                \end{customlegend}
            \end{tikzpicture}
        \end{center}
    \end{subfigure}

    \caption{Label Repair experiment results over various combinations of datasets (\vtrainsize\xspace samples) and \vxpipetype\xspace pipelines. We optimize for \vrepairgoal\xspace. The model is \vxmodel\xspace.}
    \label{fig:exp-\vscenario-\vrepairgoal-\vxpipetype-\vmodel-\vtrainsize}
\end{figure*}

\begin{figure*}
    \centering

    \def \vscenario {label-repair}
    \def \vtrainsize {1k}
    \def \vrepairgoal {fairness}
    \def \vproviders {100}
    \def \vdirtybias {0.0}
    \def \vxpipetype {fork}
    \def \vmodel {logreg}
    \def \vxmodel {logistic regression}
    
    \makeatletter
    \def \vdataset {FolkUCI}
    \@for\vpipeline:={\vucipipelines}\do{
        \@for\vutility:={acc,eqodds}\do{
            \begin{subfigure}[b]{.49\linewidth}
                \def \vpath {figures/scenario=\vscenario/trainsize=\vtrainsize/repairgoal=\vrepairgoal/providers=\vproviders/model=\vmodel/dirtybias=\vdirtybias/dataset=\vdataset/pipeline=\vpipeline/utility=\vutility/report.figure.pdf}
                \includegraphics[width=\linewidth]{\vpath}
            \end{subfigure}
        }
    }
    \makeatother

    \begin{subfigure}[c][][c]{\linewidth}
        \begin{center}
            \vspace{10pt}
            \begin{tikzpicture}
                \begin{customlegend}[legend columns=-1,legend style={column sep=5pt}]
                    \addlegendimage{myblue}\addlegendentry{Random}
                    \addlegendimage{myred}\addlegendentry{DataScope}
                    \addlegendimage{myyellow}\addlegendentry{DataScope Interactive}
                    \addlegendimage{mygreen}\addlegendentry{TMC Shapley x10}
                    \addlegendimage{mypink}\addlegendentry{TMC Shapley x100}
                \end{customlegend}
            \end{tikzpicture}
        \end{center}
    \end{subfigure}

    \caption{Label Repair experiment results over various combinations of datasets (\vtrainsize\xspace samples) and \vxpipetype\xspace pipelines. We optimize for \vrepairgoal\xspace. The model is \vxmodel\xspace.}
    \label{fig:exp-\vscenario-\vrepairgoal-\vxpipetype-\vmodel-\vtrainsize}
\end{figure*}

\begin{figure*}
    \centering

    \def \vscenario {label-repair}
    \def \vtrainsize {1k}
    \def \vrepairgoal {fairness}
    \def \vproviders {100}
    \def \vdirtybias {0.0}
    \def \vxpipetype {fork}
    \def \vmodel {knn}
    \def \vxmodel {K-nearest neighbor}
    
    \makeatletter
    \def \vdataset {FolkUCI}
    \@for\vpipeline:={\vucipipelines}\do{
        \@for\vutility:={acc,eqodds}\do{
            \begin{subfigure}[b]{.49\linewidth}
                \def \vpath {figures/scenario=\vscenario/trainsize=\vtrainsize/repairgoal=\vrepairgoal/providers=\vproviders/model=\vmodel/dirtybias=\vdirtybias/dataset=\vdataset/pipeline=\vpipeline/utility=\vutility/report.figure.pdf}
                \includegraphics[width=\linewidth]{\vpath}
            \end{subfigure}
        }
    }
    \makeatother

    \begin{subfigure}[c][][c]{\linewidth}
        \begin{center}
            \vspace{10pt}
            \begin{tikzpicture}
                \begin{customlegend}[legend columns=-1,legend style={column sep=5pt}]
                    \addlegendimage{myblue}\addlegendentry{Random}
                    \addlegendimage{myred}\addlegendentry{DataScope}
                    \addlegendimage{myyellow}\addlegendentry{DataScope Interactive}
                    \addlegendimage{mygreen}\addlegendentry{TMC Shapley x10}
                    \addlegendimage{mypink}\addlegendentry{TMC Shapley x100}
                \end{customlegend}
            \end{tikzpicture}
        \end{center}
    \end{subfigure}

    \caption{Label Repair experiment results over various combinations of datasets (\vtrainsize\xspace samples) and \vxpipetype\xspace pipelines. We optimize for \vrepairgoal\xspace. The model is \vxmodel\xspace.}
    \label{fig:exp-\vscenario-\vrepairgoal-\vxpipetype-\vmodel-\vtrainsize}
\end{figure*}

\begin{figure*}
    \centering

    \def \vscenario {label-repair}
    \def \vtrainsize {1k}
    \def \vrepairgoal {fairness}
    \def \vproviders {100}
    \def \vdirtybias {0.0}
    \def \vxpipetype {fork}
    \def \vmodel {xgb}
    \def \vxmodel {XGBoost}
    
    \makeatletter
    \def \vdataset {FolkUCI}
    \@for\vpipeline:={\vucipipelines}\do{
        \@for\vutility:={acc,eqodds}\do{
            \begin{subfigure}[b]{.49\linewidth}
                \def \vpath {figures/scenario=\vscenario/trainsize=\vtrainsize/repairgoal=\vrepairgoal/providers=\vproviders/model=\vmodel/dirtybias=\vdirtybias/dataset=\vdataset/pipeline=\vpipeline/utility=\vutility/report.figure.pdf}
                \includegraphics[width=\linewidth]{\vpath}
            \end{subfigure}
        }
    }
    \makeatother

    \begin{subfigure}[c][][c]{\linewidth}
        \begin{center}
            \vspace{10pt}
            \begin{tikzpicture}
                \begin{customlegend}[legend columns=-1,legend style={column sep=5pt}]
                    \addlegendimage{myblue}\addlegendentry{Random}
                    \addlegendimage{myred}\addlegendentry{DataScope}
                    \addlegendimage{myyellow}\addlegendentry{DataScope Interactive}
                    \addlegendimage{mygreen}\addlegendentry{TMC Shapley x10}
                    \addlegendimage{mypink}\addlegendentry{TMC Shapley x100}
                \end{customlegend}
            \end{tikzpicture}
        \end{center}
    \end{subfigure}

    \caption{Label Repair experiment results over various combinations of datasets (\vtrainsize\xspace samples) and \vxpipetype\xspace pipelines. We optimize for \vrepairgoal\xspace. The model is \vxmodel\xspace.}
    \label{fig:exp-\vscenario-\vrepairgoal-\vxpipetype-\vmodel-\vtrainsize}
\end{figure*}

%% file: proofs.tex
\section{Proofs and Details}

\subsection{Proof of \autoref{thm:decision-diagram}} \label{sec:apx-decision-diagram-proof}

\inlinesection{Model Counting for ADD's.} We start off by proving that \autoref{eq:dd-count-recursion} correctly performs model counting.

\begin{lemma} \label{lem:add-count-correct}
For a given node $n \in \mathcal{N}$ of an ADD and a given value $e \in \mathcal{E}$, \autoref{eq:dd-count-recursion} correctly computes $\mathrm{count}_e (n)$ which returns the number of assignments $v \in \mathcal{V}_{A}$ such that $\mathrm{eval}_v (n) = e$. Furthermore, when computing $\mathrm{count}_e (n)$ for any $n \in \mathcal{N}$, the number of computational steps is bounded by $O(|\mathcal{N}| \cdot |\mathcal{E}|)$.
\end{lemma}

\begin{proof}
We will prove this by induction on the structure of the recursion.

\emph{(Base case.)} Based on \autoref{eq:dd-eval-definition}, when $n = \boxdot$ we get $\mathrm{eval}_v(n) = 0$ for all $v$. Furthermore, when $n = \boxdot$, the set $\mathcal{V}_{A}[a_{>\pi(a(n))}=0]$ contains only one value assignment with all variables set to zero. Hence, the model count will equal to $1$ only for $e=0$ and it will be $0$ otherwise, which is reflected in the base cases of \autoref{eq:dd-count-recursion}.

\emph{(Inductive step.)} Because our ADD is ordered and full, both $c_L(n)$ and $c_H(n)$ are associated with the same variable, which is the predecessor of $a(n)$ in the permutation $\pi$. Based on this and the induction hypothesis, we can assume that
\begin{gather} \label{eq:count-proof-induction-components}
    \begin{split}
        \mathrm{count}_{e - w_L(n)} (c_L(n)) &= \Big|\Big\{ v \in \mathcal{V}_{A [\leq a(c_L(n))] } \ | \ \mathrm{eval}_v (c_L(n)) = e - w_L(n) \Big\}\Big| \\
        \mathrm{count}_{e - w_H(n)} (c_H(n)) &= \Big|\Big\{ v \in \mathcal{V}_{A [\leq a(c_H(n))] } \ | \ \mathrm{eval}_v (c_H(n)) = e - w_H(n) \Big\}\Big|
    \end{split}
\end{gather}
We would like to compute $\mathrm{count}_e (n)$ as defined in \autoref{eq:dd-count-definition}. It computes the size of a set defined over possible value assignments to variables in $A [\leq a(n)]$. The set of value assignments can be partitioned into two distinct sets: one where $a(n) \gets 0$ and one where $a(n) \gets 1$. We thus obtain the following expression:
\begin{align}
    \begin{split}
        \mathrm{count}_e (n) :=
        & \Big|\Big\{ v \in \mathcal{V}_{A [\leq a(n)]} \big[ a(n) \gets 0 \big] \ | \ \mathrm{eval}_v (n) = e \Big\}\Big| \\
        + \ 
        & \Big|\Big\{ v \in \mathcal{V}_{A [\leq a(n)]} \big[ a(n) \gets 1 \big] \ | \ \mathrm{eval}_v (n) = e \Big\}\Big|
    \end{split}
\end{align}
Based on \autoref{eq:dd-eval-definition}, we can transform the $\mathrm{eval}_v (n)$ expressions as such:
\begin{align}
    \begin{split}
        \mathrm{count}_e (n) :=
        & \Big|\Big\{ v \in \mathcal{V}_{A [\leq a(c_L(n))]} \ | \ w_L(n) + \mathrm{eval}_v (c_L(n)) = e \Big\}\Big| \\
        + \ 
        & \Big|\Big\{ v \in \mathcal{V}_{A [\leq a(c_L(n))]} \ | \ w_H(n) + \mathrm{eval}_v (c_H(n)) = e \Big\}\Big|
    \end{split}
\end{align}
Finally, we can notice that the set size expressions are equivalent to those in \autoref{eq:count-proof-induction-components}. Therefore, we can obtain the following expression:
\begin{equation}
    \mathrm{count}_e (n) := \mathrm{count}_{e - w_L(n)} (c_L(n)) + \mathrm{count}_{e - w_H(n)} (c_H(n))
\end{equation}
which is exactly the recursive step in \autoref{eq:dd-count-recursion}. This concludes our inductive proof and we move onto proving the complexity bound.

\emph{(Complexity.)} This is trivially proven by observing that since $\mathrm{count}$ has two arguments, we can maintain a table of results obtained for each $n \in \mathcal{N}$ and $e \in \mathcal{E}$. Therefore, we know that we will never need to perform more than $O(|\mathcal{N}| \cdot |\mathcal{E}|)$ invocations of $\mathrm{count}_e (n)$.

\end{proof}

\inlinesection{ADD Construction.} Next, we prove that the size of an ADD resulting from \emph{diagram summation} as defined in \autoref{sec:additive-decision-diagrams} is linear in the number of variables.

The size of the diagram resulting from a sum of two diagrams with node sets $\mathcal{N}_1$ and $\mathcal{N}_2$ can be loosely bounded by $O(|\mathcal{N}_1| \cdot |\mathcal{N}_2|)$ assuming that its nodes come from a combination of every possible pair of operand nodes. However, given the much more narrow assumptions we made in the definition of the node sum operator, we can make this bound considerably tighter. For this we define the \emph{diameter} of an ADD as the maximum number of nodes associated with any single variable. Formally we can write:
\begin{equation}
    \mathrm{diam}(\mathcal{N}) := \max_{a_i \in A} \big| \{ n \in \mathcal{N} : a(n) = a_i \} \big|
\end{equation}
We can immediately notice that the size of any ADD with set of nodes $\mathcal{N}$ and variables $A$ is bounded by $O(|A| \cdot \mathrm{diam}(\mathcal{N}))$. We can use this fact to prove a tighter bound on the size of an ADD resulting from a sum operation:

\begin{lemma}
Given two full ordered ADD's with nodes $\mathcal{N}_1$ and $\mathcal{N}_2$, noth defined over the set of variables $A$, the number of nodes in $\mathcal{N}_1 + \mathcal{N}_2$ is bounded by $O(|A| \cdot \mathrm{diam}(\mathcal{N}_1) \cdot \mathrm{diam}(\mathcal{N}_2))$.
\end{lemma}

\begin{proof}
It is sufficient to show that $\mathrm{diam} (\mathcal{N}_1 + \mathcal{N}_2) = O(\mathrm{diam} (\mathcal{N}_1) \cdot \mathrm{diam} (\mathcal{N}_2))$. This is a direct consequence of the fact that for full ordered ADD's the node sum operator is defined only for nodes associated with the same variable. Since the only way to produce new nodes is by merging one node in $\mathcal{N}_1$ with one node in $\mathcal{N}_2$, and given that we can merge nodes associated with the same variable, the number of nodes associated with the same variable in the resulting ADD equals the product of the corresponding number of nodes in the constituent ADD's. Since the diameter is simply the upper bound of the number of nodes associated with any single variable, the same upper bound in the resulting ADD cannot be larger than the product of the upper bounds of constituent nodes.
\end{proof}

\inlinesection{Computing the Oracle using ADD's.} Finally, we prove the correctness of \autoref{thm:decision-diagram}.

\begin{lemma} \label{lem:oracle-as-add-count}
Given an Additive Decision diagram with root node $n_{t, t'}$ that computes the Boolean function $\phi_{t, t'}(v)$ as defined in \autoref{eq:oracle-add-function}, the counting oracle $\omega_{t, t'} (\alpha, \gamma, \gamma')$ defined in \autoref{eq:counting-oracle} can be computed as:
\begin{equation}
    \omega_{t, t'} (\alpha, \gamma, \gamma') := \mathrm{count}_{(\alpha, \gamma, \gamma')} (n_{t, t'})
\end{equation}
\end{lemma}
\begin{proof}

Let us define $\mathcal{D}[\geq_\sigma t] \subseteq \mathcal{D}$ as a set of tuples with similarity higher or equal than that of $t$, formally $\mathcal{D}[\geq_\sigma t] := \{ t' \in \mathcal{D} : \sigma(t') \geq \sigma(t) \}$. Similarly to $\mathcal{D}$, the semantics of $\mathcal{D}[\geq_\sigma t]$ is also that of a set of possible candidate sets. Given a value assignment $v$, we can obtain $\mathcal{D}[\geq_\sigma t][v]$ from $\mathcal{D}[v]$. For convenience, we also define $\mathcal{D}[\geq_\sigma t][y]$ as a subset of $\mathcal{D}[\geq_\sigma t]$ with only tuples that have label $y$. Given these definitions, we can define several equivalences. First, for $\mathrm{top}_K$ we have:
\begin{equation}
    \Big( t = \mathrm{top}_K \mathcal{D}[v] \Big) \iff
    \Big(t \in \mathcal{D}[v] \wedge \big| \mathcal{D}[\geq_\sigma t][v] \big| = K \Big)
\end{equation}
In other words, for $t$ to be the tuple with the $K$-th highest similarity in $\mathcal{D}[v]$, it needs to be a member of $\mathcal{D}[v]$ and the number of tuples with similarity greater or equal to $t$ has to be exactly $K$. Similarly, we can define the equivalence for $\mathrm{tally}_{t}$:
\begin{equation} \label{eq:tally-dataset-equivalence-step}
    \Big( \gamma = \mathrm{tally}_t \mathcal{D}[v] \Big) \iff
    \Big( \forall y \in \mathcal{Y}, \gamma_{y} = \big| \mathcal{D}[ \geq_\sigma t][ y][ v] \big|  \Big)
\end{equation}
This is simply an expression that partitions the set $\mathcal{D}[\geq_\sigma t][v]$ based on $y$ and tallies them up. The next step is to define an equivalence for $( t = \mathrm{top}_K \mathcal{D}[v] ) \wedge ( \gamma = \mathrm{tally}_t \mathcal{D}[v] )$. We can notice that since $|\gamma| = K$, if we have $( \forall y \in \mathcal{Y}, \gamma_{y} = |\mathcal{D}[\geq_\sigma t][y][v]| )$ then we can conclude that $(|\mathcal{D}[\geq_\sigma t][v]| = K)$ is redundant. Hence, we can obtain:
\begin{equation}
    \Big( t = \mathrm{top}_K \mathcal{D}[v] \Big) \wedge
    \Big( \gamma = \mathrm{tally}_t \mathcal{D}[v] \Big) \iff
    \Big(t \in \mathcal{D}[v] \Big) \wedge
    \Big( \forall y \in \mathcal{Y}, \gamma_{y} = \big| \mathcal{D}[\geq_\sigma t][y][v] \big|  \Big)
\end{equation}
According to \autoref{eq:tally-dataset-equivalence-step}, we can reformulate the right-hand side of the above equivalence as:
\begin{equation}
    \Big( t = \mathrm{top}_K \mathcal{D}[v] \Big) \wedge
    \Big( \gamma = \mathrm{tally}_t \mathcal{D}[v] \Big) \iff
    \Big(t \in \mathcal{D}[v] \Big) \wedge
    \Big( \gamma = \mathrm{tally}_t \mathcal{D}[v] \Big)
\end{equation}
We can construct a similar expression for $t'$ and $v[a_i = 1]$ so we cover four out of five predicates in \autoref{eq:counting-oracle}. The remaining one is simply the support of the value assignment $v$ which we will leave intact. This leads us with the following equation for the counting oracle:
\begin{gather}
\begin{split}
    \omega_{t, t'} (\alpha, \gamma, \gamma') :=
    \sum_{v \in \mathcal{V}_{A}[a_i \gets 0]}
    & \mathbbm{1} \{ \alpha = |\mathrm{supp}(v)| \} \\
    & \mathbbm{1} \{ t \in f(\mathcal{D}[v]) \}
        \mathbbm{1} \{ t' \in f(\mathcal{D}[v[a_i \gets 1]]) \} \\
    & \mathbbm{1} \{ \gamma = \mathrm{tally}_t \mathcal{D}[v] \}
        \mathbbm{1} \{ \gamma = \mathrm{tally}_t \mathcal{D}[v[a_i \gets 1]] \}
\end{split} \label{eq:counting-oracle-dd-redef}
\end{gather}
We can use the Boolean function $\phi_{t, t'}(v)$ in \autoref{eq:oracle-add-function} to simplify the above equation. Notice that the conditions $t \in f(\mathcal{D}[v])$ and $t' \in f(\mathcal{D}[v[a_i \gets 1]])$ are embedded in the definition of $\phi_{t, t'}(v)$ which will return $\infty$ if those conditions are not met. When the conditions are met, $\phi_{t, t'}(v)$ returns exactly the same triple $(\alpha, \gamma, \gamma')$. Therefore it is safe to replace the five indicator functions in the above formula with a single one as such:
\begin{equation}
    \omega_{t, t'} (\alpha, \gamma, \gamma') :=
    \sum_{v \in \mathcal{V}_{A}[a_i \gets 0]}
    \mathbbm{1} \{ (\alpha, \gamma, \gamma') = \phi_{t, t'}(v) \}
\end{equation}
Given our assumption that $\phi_{t, t'}(v)$ can be represented by an ADD with a root node $n_{t, t'}$, the above formula is exactly the model counting operation:
\begin{equation}
    \omega_{t, t'} (\alpha, \gamma, \gamma') := \mathrm{count}_{(\alpha, \gamma, \gamma')} (n_{t, t'})
\end{equation}

\end{proof}

\begin{theorem} \label{thm:decision-diagram-appendix}
If we can represent the Boolean function $\phi_{t, t'}(v)$ defined in \autoref{eq:oracle-add-function} with an Additive Decision Diagram of size polynomial in $|\mathcal{D}|$ and $|f(\mathcal{D})|$, then we can compute the counting oracle $\omega_{t, t'}$ in time polynomial in $|\mathcal{D}|$ and $|f(\mathcal{D})|$.
\end{theorem}

\begin{proof}
This theorem follows from the two previously proved lemmas: \autoref{lem:add-count-correct} and \autoref{lem:oracle-as-add-count}. Namely, as a result of \autoref{lem:oracle-as-add-count} we claim that model counting of the Boolean function $\phi_{t, t'}(v)$ is equivalent to computing the oracle result. On top of that, as a result of \autoref{lem:add-count-correct} we know that we can perform model counting in time linear in the size of the decision diagram. Hence, if our function $\phi_{t, t'}(v)$ can be represented with a decision diagram of size polynomial in the size of data, then we can conclude that computing the oracle result can be done in time polynomial in the size of data.
\end{proof}

\subsection{Proof of corollary \autoref{col:complexity-knn-join}} \label{sec:apx-complexity-knn-join-proof}

\begin{proof}
    This follows from the observation that in \autoref{alg:compile-dataset-to-add}, each connected component $A_C$ will be made up from one variable corresponding to the dimension table and one or more variables corresponding to the fact table. Since the fact table variables will be categorized as "leaf variables", the expression $A_C \setminus A_L$ in Line \ref{alg:cmp:line:add-tree} will contain only a single element -- the dimension table variable. Consequently, the ADD tree in $\mathcal{N}'$ will contain a single node. On the other side, the $A_C \cap A_L$ expression will contain all fact table variables associated with that single dimension table variable. That chain will be added to the ADD tree two times for two outgoing branches of the single tree node. Hence, the ADD segment will be made up of two fact table variable chains stemming from a single dimension table variable node. There will be $O(|\mathcal{D}_D|)$ partitions in total. Given that the fact table variables are partitioned, the cumulative size of their chains will be $O(|\mathcal{D}_F|)$. Therefore, the total size of the ADD with all partitions joined together is bounded by $O(|\mathcal{D}_D|+|\mathcal{D}_F|) = O(N)$.
    
    Given fact and combining it with \autoref{thm:decision-diagram} we know that the counting oracle can be computed in time $O(N)$ time. Finally, given \autoref{thm:shapley-using-counting-oracle} and the structure of \autoref{eq:shap-main} we can observe that the counting oracle is invoked $O(N^3)$ times. As a result, we can conclude that the total complexity of computing the Shapley value is $O(N^4)$.
\end{proof}

\subsection{Proof of corollary \autoref{col:complexity-knn-fork}} \label{sec:apx-complexity-knn-fork-proof}

\begin{proof}
    The key observation here is that, since all provenance polynomials contain only a single variable, there is no interdependency between them, which means that the connected components returned in Line \ref{alg:cmp:line:conn-cmp} of \autoref{alg:compile-dataset-to-add} will each contain a single variable. Therefore, the size of the resulting ADD will be $O(N)$. Consequently, similar to the proof of the previous corollary, the counting oracle can be computed in time $O(N)$ time. In this case, the size of the output dataset is $O(M)$ which means that \autoref{eq:shap-main} willinvoke the oracle $O(M^2 N)$ times. Therefore, the total time complexity of computing the Shapley value will be $O(M^2 N^2)$.
\end{proof}

\subsection{Proof of corollary \autoref{col:complexity-knn-map}} \label{sec:apx-complexity-knn-map-proof}

\begin{proof}
    There are two arguments we need to make which will result in the reduction of complexity compared to fork pipelines. The first argument is that, given that each variable can appear in the provenance polynomial of at most one tuple, having its value set to $1$ can result in either zero or one tuple contributing to the top-$K$ tally. It will be one if that tuple is more similar than the boundary tuple $t$ and it will be zero if it is less similar. Consequently, our ADD will have a chain structure with high-child increments being either $0$ or $1$. If we partition the ADD into two chains, one with all increments $1$ and another with all increments $0$, then we end up with two uniform ADD's. As shown in \autoref{eq:dd-count-uniform}, model counting of uniform ADD's can be achieved in constant time. The only difference here is that, since we have to account for the support size each model, computing the oracle $\omega_{t, t'} (\alpha, \gamma, \gamma')$ for a given $\alpha$ will require us to account for different possible ways to split $\alpha$ across the two ADD's. However, since the tuple $t$ needs to be the boundary tuple, which means it is the $K$-th most similar, there need to be exactly $K-1$ variables from the ADD with increments $1$ that can be set to $1$. This gives us a single possible distribution of $\alpha$ across two ADD's. Hence, the oracle can be computed in constant time.
    
    As for the second argument, we need to make a simple observation. For map pipelines, given a boundary tuple $t$ and a tally vector $\gamma$ corresponding to the variable $a_i$ being assigned the value $0$, we know that setting this variable to $1$ can introduce at most one tuple to the top-$K$. That could only be the single tuple associated with $a_i$. If this tuple has a lower similarity score than $t$, there will be no change in the top-$K$. On the other side, if it has a higher similarity, then it will become part of the top-$K$ and it will evict exactly $t$ from it. Hence, there is a unique tally vector $\gamma'$ resulting from $a_i$ being assigned the value $1$. This means that instead of computing the counting oracle $\omega_{t, t'} (\alpha, \gamma, \gamma')$, we can compute the oracle $\omega_t (\alpha, \gamma)$. This means that, in \autoref{eq:shap-main} we can eliminate the iteration over $t'$ which saves us an order of $O(N)$ in complexity.
    
    As a result, \autoref{eq:shap-main} will make $O(N^2)$ invocations to the oracle which can be computed in constant time. Hence, the final complexity of computing the Shapley value will be $O(N^2)$.
\end{proof}

\subsection{Proof of corollary \autoref{col:complexity-1nn-map}} \label{sec:apx-complexity-1nn-map-proof}

\begin{proof}
    We start off by plugging in the oracle definition from \autoref{eq:oracle-1nn} into the Shapley value computation \autoref{eq:shap-main}:
    \begin{gather}
    \begin{split}
        \varphi_i = \frac{1}{N}
        \sum_{t, t' \in f(\mathcal{D})}
        \sum_{\alpha=1}^{N}
        \binom{N-1}{\alpha}^{-1}
        \sum_{\gamma, \gamma' \in \Gamma}
        u_{\Delta} (\gamma, \gamma')
        & \binom{|\{t'' \in \mathcal{D} \ : \ \sigma(t'') < \sigma(t) \}|}{\alpha} \\
        & \mathbbm{1} \{ p(t')=a_i \} \\
        & \mathbbm{1} \{ \gamma = \Gamma_{y(t)} \}
        \mathbbm{1} \{ \gamma' = \Gamma_{y(t')} \} 
    \end{split}
    \end{gather}
    As we can see, the oracle imposes hard constraints on the tuple $t'$ and tally vectors $\gamma$ and $\gamma'$. We will replace the tally vectors with their respective constants and the tuple $t'$ we will denote as $t_i$ because it is the only tuple associated with $a_i$. Because of this, we can remove the sums that iterate over them:
    \begin{gather}
        \varphi_i = \frac{1}{N}
        \sum_{t \in f(\mathcal{D})}
        \sum_{\alpha=1}^{N}
        \binom{N-1}{\alpha}^{-1}
        u_{\Delta} (\Gamma_{y(t)}, \Gamma_{y(t_i)})
        \binom{|\{t'' \in \mathcal{D} \ : \ \sigma(t'') < \sigma(t) \}|}{\alpha}
    \end{gather}
    We could significantly simplify this equation by assuming the tuples in $f(\mathcal{D})$ are sorted by decreasing similarity. We then obtain:
    \begin{equation}
        \varphi_i = \frac{1}{N}
        \sum_{j = 1}^{N}
        \sum_{\alpha=1}^{N}
        \binom{N-1}{\alpha}^{-1}
        u_{\Delta} (\Gamma_{y(t)}, \Gamma_{y(t_i)})
        \binom{N-j}{\alpha}
    \end{equation}
    We shuffle the sums a little by multiplying $\frac{1}{N}$ with $\binom{N-1}{\alpha}^{-1}$ and expanding the $u_{\Delta}$ function according to its definition. We also alter the limit of the innermost sum because $\alpha \leq N - j$. Thus, we obtain:
    \begin{equation}
        \varphi_i =
        \sum_{j = 1}^{N}
        \Big(
            \mathbbm{1} \{ y(t_i) = y(t_v) \} -
            \mathbbm{1} \{ y(t_j) = y(t_v) \}
        \Big)
        \sum_{\alpha=1}^{N-j}
        \binom{N}{\alpha}^{-1}
        \binom{N-j}{\alpha}
    \end{equation}
    The inntermost sum in the above equation can be simplified by applying the so-called Hockey-stick identity \cite{ross1997generalized}. Specifically, $\binom{N}{\alpha}^{-1}\binom{N-j}{\alpha}$ becomes $\binom{N}{j}^{-1}\binom{N-\alpha}{j}$. Then, $\sum_{\alpha=1}^{N-j}\binom{N}{j}^{-1}\binom{N-\alpha}{j}$ becomes $\binom{N}{j}^{-1}\binom{N}{j+1}$. Finally, we obtain the following formula:
    \begin{equation}
        \varphi_i =
        \sum_{j = 1}^{N}
        \Big(
            \mathbbm{1} \{ y(t_i) = y(t_v) \} -
            \mathbbm{1} \{ y(t_j) = y(t_v) \}
        \Big)
        \binom{N-j}{j+1}
    \end{equation}
    As we can see, the above formula can be computed in $O(N)$ iterations. Therefore, given that we still need to sort the dataset beforehand, the ovarall complexity of the entire Shapley value amounts to $O(N \log N)$.
\end{proof}

\subsection{Proof of corollary \autoref{col:complexity-1nn-fork}} \label{sec:apx-complexity-1nn-fork-proof}

\begin{proof}
    We will prove this by reducing the problem of Shapley value computation in fork pipelines to the one of computing it for map pipelines. Let us have two tuples $t_{j,1}, t_{j,2} \in f(\mathcal{D})$, both associated with some variable $a_j \in A$. That means that $p(t_{j,1}) = p(t_{j,2})$. If we examine \autoref{eq:top-1-condition-map-single}, we notice that it will surely evaluate to false if either $\sigma(t_{j,1}) > \sigma(t)$ or $\sigma(t_{j,2}) > \sigma(t)$. The same observation holds for \autoref{eq:top-1-condition-map}.
    
    Without loss of generality, assume $\sigma(t_{j,1}) > \sigma(t_{j,2})$. Then, $\sigma(t_{j,1}) > \sigma(t)$ implies $\sigma(t_{j,2}) > \sigma(t)$. As a result, we only ever need to check the former condition without paying attention to the latter. The outcome of this is that for all sets of tuples associated with the same variable, it is safe to ignore all of them except the one with the highest similarity score, and we will nevertheless obtain the same oracle result. Since we transformed the problem to the one where for each variable we have to consider only a single associated tuple, then we have effectively reduced the problem to the one of computing Shapley value for map pipelines. Consequently, we can apply the same algorithm and will end up with the same time complexity.
\end{proof}

\subsection{Details of \autoref{alg:compile-dataset-to-add}} \label{sec:apx-alg-compile-dataset-to-add-details}

In this section we examine the method of compiling a provenance-tracked dataset $f(\mathcal{D})$ that results from a pipeline $f$. The crux of the method is defined in \autoref{alg:compile-dataset-to-add} which is an algorithm that takes a dataset $\mathcal{D}$ with provenance tracked over a set of variables $\mathcal{X}$ and a boundary tuple $t \in \mathcal{D}$. The result is an ADD that computes the following function:
\begin{equation} \label{eq:oracle-add-function-single}
    \phi_{t}(v) := \begin{cases}
        \infty,     & \mathrm{if} \ t \not\in \mathcal{D}[v], \\
        \mathrm{tally}_t \mathcal{D}[v]    & \mathrm{otherwise}. \\
    \end{cases}
\end{equation}
Assuming that all provenance polynomials are actually a single conjunction of variables, and that the tally is always a sum over those polynomials, it tries to perform factoring by determining if there are any variables that can be isolated. This is achieved by first extracting variables that appear only once (Line \ref{alg:cmp:line:leaves}) separating the total sum into components that don't share any variables (Line \ref{alg:cmp:line:conn-cmp}). Then for the variables that cannot be isolated (because they appear in polynomials in multiple tuples with multiple different variables) we form a group which will be treated as one binary vector and based on the value of that vector we would take a specific path in the tree. We thus take the group of variables and call the \textsc{\small ConstructADDTree} function to construct an ADD tree (Line \ref{alg:cmp:line:add-tree}).

Every path in this tree corresponds to one value assignment to the variables in that tree. Then, for every path we call the \textsc{\small ConstructADDChain} to build a chain made up of the isolated variables and call \textsc{\small AppendToADDPath} to append them to the leaf of that path (Line \ref{alg:cmp:line:append-path}). For each variable in the chain we also define an increment that is defined by the number of tuples that will be more similar than the boundary tuple $t$ and also have their provenance polynomial "supported" by the path. We thus construct a segment of the final ADD made up of different components. We append this segment to the final ADD using the \textsc{\small AppendToADDRoot} function. We don't explicitly define these functions but we illustrate their functionality in \autoref{fig:example-add-compilation-functions}.

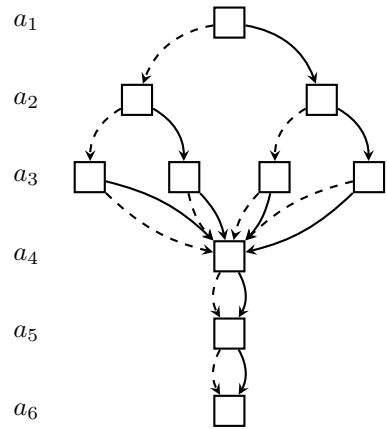
\begin{figure*}[!ht]
    \centering
    \begin{tikzpicture}[align=center, node distance=6mm and 2mm, line width=0.8pt] 
        \tikzstyle{invisible} = [minimum width=0, minimum height=4.5mm]
        \tikzstyle{free} = [inner sep=5pt]
        \tikzstyle{var} = [draw, rectangle, inner sep=2pt, minimum width=4mm, minimum height=4mm]
        \tikzstyle{root} = [draw, rectangle, inner sep=2pt, minimum width=4mm, minimum height=4mm]
        
        \tikzstyle{path} = [draw=myred]
        
        \begin{scope}[local bounding box=n1]
        
            
            \node[var] (n11) {};
            
            \draw (n11 -| -3,0) node[anchor=west] {$a_{1}$};
            
            
            \node[var] (n12) [below left=6mm and 8mm of n11] {};
            \draw[-stealth, dashed] (n11) to [bend right] (n12);
            
            \node[var] (n13) [below right=6mm and 8mm of n11] {};
            \draw[-stealth] (n11) to [bend left] (n13);
            
            \draw (n12 -| -3,0) node[anchor=west] {$a_{2}$};
            
            
            \node[var] (n14) [below left=of n12] {};
            \draw[-stealth, dashed] (n12) to [bend right] (n14);
            
            \node[var] (n15) [below right=of n12] {};
            \draw[-stealth] (n12) to [bend left] (n15);
            
            \node[var] (n16) [below left=of n13] {};
            \draw[-stealth, dashed] (n13) to [bend right] (n16);
            
            \node[var] (n17) [below right=of n13] {};
            \draw[-stealth] (n13) to [bend left] (n17);
            
            \draw (n14 -| -3,0) node[anchor=west] {$a_{3}$};
            
            
            \node[invisible] (n15n16) [shift={($(n15.east)!0.5!(n16.west)$)}] {};

        \end{scope}
        
        \node[free] (n1-cap) [above=of n11] {\small $\mathcal{N}_1 = $ \textsc{ConstructADDTree}$(\{a_1, a_2, a_3\})$};
        
        \begin{scope}[local bounding box=n2, shift={($(n11.east)+(8cm,0)$)}]
        
            
            \node[var] (n21) {};
            \draw (n21 -| -0.9,0) node[anchor=west] {$a_{4}$};
            
            
            \node[var] (n22) [below=of n21] {};
            \draw[-stealth, dashed] (n21) to [bend right] (n22);
            \draw[-stealth] (n21) to [bend left] (n22);
            \draw (n22 -| -0.9,0) node[anchor=west] {$a_{5}$};
            
            
            \node[var] (n23) [below=of n22] {};
            \draw[-stealth, dashed] (n22) to [bend right] (n23);
            \draw[-stealth] (n22) to [bend left] (n23);
            \draw (n23 -| -0.9,0) node[anchor=west] {$a_{6}$};
            
            
        
        \end{scope}
        
        \node[free] (n2-cap) [above=of n21] {\small $\mathcal{N}_2 = $ \textsc{ConstructADDChain}$(\{a_4, a_5, a_6\})$};
        
        \begin{scope}[local bounding box=n3, shift={($(n15n16.south)+(0,-2cm)$)}]
        
            
            \node[var] (n31) {};
            
            \draw (n31 -| -3,0) node[anchor=west] {$a_{1}$};
            
            
            \node[var] (n32) [below left=6mm and 8mm of n31] {};
            \draw[-stealth, dashed] (n31) to [bend right] (n32);
            
            \node[var] (n33) [below right=6mm and 8mm of n31] {};
            \draw[-stealth] (n31) to [bend left] (n33);
            
            \draw (n32 -| -3,0) node[anchor=west] {$a_{2}$};
            
            
            \node[var] (n34) [below left=of n32] {};
            \draw[-stealth, dashed] (n32) to [bend right] (n34);
            
            \node[var] (n35) [below right=of n32] {};
            \draw[-stealth] (n32) to [bend left] (n35);
            
            \node[var] (n36) [below left=of n33] {};
            \draw[-stealth, dashed] (n33) to [bend right] (n36);
            
            \node[var] (n37) [below right=of n33] {};
            \draw[-stealth] (n33) to [bend left] (n37);
            
            \draw (n34 -| -3,0) node[anchor=west] {$a_{3}$};
            
            
            \node[var] (n38) [below right=of n36] {};
            \draw[-stealth] (n36) to [bend left] (n38);
            \draw (n38 -| -3,0) node[anchor=west] {$a_{4}$};
            
            
            \node[var] (n39) [below=of n38] {};
            \draw[-stealth, dashed] (n38) to [bend right] (n39);
            \draw[-stealth] (n38) to [bend left] (n39);
            \draw (n39 -| -3,0) node[anchor=west] {$a_{5}$};
            
            
            \node[var] (n310) [below=of n39] {};
            \draw[-stealth, dashed] (n39) to [bend right] (n310);
            \draw[-stealth] (n39) to [bend left] (n310);
            \draw (n310 -| -3,0) node[anchor=west] {$a_{6}$};
            
            
        
        \end{scope}
        
        \node[free] (n3-cap) [above=of n31] {\small \textsc{AppendToADDPath}$(\mathcal{N}_1, \mathcal{N}_2, \{ a_1 \rightarrow 1, a_2 \rightarrow 0, a_3 \rightarrow 1 \})$};
        
        \begin{scope}[local bounding box=n3, shift={($(n31.east)+(8cm,0)$)}]
        
            
            \node[var] (n41) {};
            
            \draw (n41 -| -3,0) node[anchor=west] {$a_{1}$};
            
            
            \node[var] (n42) [below left=6mm and 8mm of n41] {};
            \draw[-stealth, dashed] (n41) to [bend right] (n42);
            
            \node[var] (n43) [below right=6mm and 8mm of n41] {};
            \draw[-stealth] (n41) to [bend left] (n43);
            
            \draw (n42 -| -3,0) node[anchor=west] {$a_{2}$};
            
            
            \node[var] (n44) [below left=of n42] {};
            \draw[-stealth, dashed] (n42) to [bend right] (n44);
            
            \node[var] (n45) [below right=of n42] {};
            \draw[-stealth] (n42) to [bend left] (n45);
            
            \node[var] (n46) [below left=of n43] {};
            \draw[-stealth, dashed] (n43) to [bend right] (n46);
            
            \node[var] (n47) [below right=of n43] {};
            \draw[-stealth] (n43) to [bend left] (n47);
            
            \draw (n44 -| -3,0) node[anchor=west] {$a_{3}$};
            
            
            \node[invisible] (n45n46) [shift={($(n45.east)!0.5!(n46.west)$)}] {};
            \node[var] (n48) [below=of n45n46] {};
            \draw[-stealth] (n44) to [bend left=15] (n48);
            \draw[-stealth, dashed] (n44) to [bend right=15] (n48);
            \draw[-stealth] (n45) to [bend left=15] (n48);
            \draw[-stealth, dashed] (n45) to [bend right=15] (n48);
            \draw[-stealth] (n46) to [bend left=15] (n48);
            \draw[-stealth, dashed] (n46) to [bend right=15] (n48);
            \draw[-stealth] (n47) to [bend left=15] (n48);
            \draw[-stealth, dashed] (n47) to [bend right=15] (n48);
            
            \draw (n48 -| -3,0) node[anchor=west] {$a_{4}$};
            
            
            \node[var] (n49) [below=of n48] {};
            \draw[-stealth, dashed] (n48) to [bend right] (n49);
            \draw[-stealth] (n48) to [bend left] (n49);
            \draw (n49 -| -3,0) node[anchor=west] {$a_{5}$};
            
            
            \node[var] (n410) [below=of n49] {};
            \draw[-stealth, dashed] (n49) to [bend right] (n410);
            \draw[-stealth] (n49) to [bend left] (n410);
            \draw (n410 -| -3,0) node[anchor=west] {$a_{6}$};
            
            
        
        \end{scope}
        
        \node[free] (n4-cap) [above=of n41] {\small \textsc{AppendToADDRoot}$(\mathcal{N}_1, \mathcal{N}_2)$};
        
        
    \end{tikzpicture}
    
    \caption{An example of a ADD compilation functions.
    }
    \label{fig:example-add-compilation-functions}
    
\end{figure*}